\documentclass{article}

\usepackage{mathptmx}       
\usepackage{helvet}         
\usepackage{courier}        
\usepackage{type1cm}        
\usepackage{verbatim}
\usepackage{makeidx}         
\usepackage{graphicx}        
\usepackage{multicol}        
\usepackage[bottom]{footmisc}
\usepackage{geometry} 

\makeindex             

\usepackage[utf8]{inputenc}
\usepackage{amsmath}
\usepackage{amsfonts}
\usepackage{dsfont}
\usepackage{enumitem}
\usepackage{amssymb}
\usepackage{xfrac}
\usepackage{amsthm} 
\newtheorem{theorem}{Theorem}[section]

\usepackage[usenames,dvipsnames]{xcolor}
\usepackage[breakable, theorems, skins]{tcolorbox}

\tcbset{enhanced}

\DeclareRobustCommand{\mybox}[2][gray!25]{%
	\vspace{0.2cm}
	\begin{tcolorbox}[   
		breakable,
		left=0pt,
		right=0pt,
		top=0pt,
		bottom=0pt,
		colback=#1,
		colframe=#1,
		width=1.01\dimexpr\textwidth\relax, 
		enlarge left by=-1mm,
		boxsep=5pt,
		arc=0pt,outer arc=0pt,
		]
		#2
	\end{tcolorbox}
}

\DeclareSymbolFont{bfletters}{OT1}{cmr}{bx}{n}
\DeclareSymbolFontAlphabet{\mathbf}{bfletters}

\DeclareMathSymbol{X}{\mathalpha}{bfletters}{`X}
\DeclareMathSymbol{L}{\mathalpha}{bfletters}{`L}
\DeclareMathSymbol{W}{\mathalpha}{bfletters}{`W}
\DeclareMathSymbol{I}{\mathalpha}{bfletters}{`I}
\DeclareMathSymbol{U}{\mathalpha}{bfletters}{`U}
\DeclareMathSymbol{D}{\mathalpha}{bfletters}{`D}
\DeclareMathSymbol{P}{\mathalpha}{bfletters}{`P}
\DeclareMathSymbol{K}{\mathalpha}{bfletters}{`K}
\DeclareMathSymbol{A}{\mathalpha}{bfletters}{`A}
\DeclareMathSymbol{B}{\mathalpha}{bfletters}{`B}
\DeclareMathSymbol{M}{\mathalpha}{bfletters}{`M}
\DeclareMathSymbol{R}{\mathalpha}{bfletters}{`R}

\DeclareMathSymbol{u}{\mathalpha}{bfletters}{`u}
\DeclareMathSymbol{x}{\mathalpha}{bfletters}{`x}
\DeclareMathSymbol{c}{\mathalpha}{bfletters}{`c}
\DeclareMathSymbol{y}{\mathalpha}{bfletters}{`y}
\DeclareMathSymbol{z}{\mathalpha}{bfletters}{`z}
\DeclareMathSymbol{p}{\mathalpha}{bfletters}{`p}

\DeclareMathOperator*{\argmin}{arg\,min}
\DeclareMathOperator*{\argmax}{arg\,max}

\renewcommand{\c}{\textit{c}}

\newcommand{\pp}{\textit{p}}
\newcommand{\trace}[1]{\text{tr}\hspace{-1px}\left(#1\right)}
\newcommand{\ncut}{\texttt{ncut}}
\newcommand{\rcut}{\texttt{rcut}}
\newcommand{\spanning}[1]{\text{span}(#1)}
\newcommand{\Rbb}{\mathbb{R}}

\renewcommand{\sf}[1]{\mathsf{#1}} 

\renewcommand{\O}{\mathcal{O}}

\newcommand{\todo}[1]{{\color{red}\textbf{TODO:} #1}}

\newtheorem{definition}[theorem]{Definition}
\newtheorem{lemma}[theorem]{Lemma}

\newtheorem{remark}[theorem]{Remark}

\title{Approximating Spectral Clustering via Sampling: \\a Review}
\author{Nicolas Tremblay{\small$^1$} and Andreas Loukas{\small$^2$}}
\date{%
	\small $^1$ CNRS, Univ. Grenoble Alpes, Grenoble-INP, GIPSA-lab, France\\%
	$^2$ Ecole Polytechnique F\'ed\'erale de Lausanne, Switzerland\\[2ex]%
}

\begin{document}

    
	
	%
	%
	\maketitle
	
	
	\begin{abstract}
	Spectral clustering refers to a family of unsupervised learning algorithms that compute a spectral embedding of the original data based on the eigenvectors of a similarity graph. This non-linear transformation of the data is both the key of these algorithms' success and their Achilles heel: forming a graph and computing its dominant eigenvectors can indeed be computationally prohibitive when dealing with more that a few tens of thousands of points. 
    In this paper, we review the principal research efforts aiming to reduce this computational cost. We focus on methods that come with a theoretical control on the clustering performance and incorporate some form of sampling in their operation. Such methods abound in the machine learning, numerical linear algebra, and graph signal processing literature and, amongst others, include Nystr\"om-approximation, landmarks, coarsening, coresets, and compressive spectral clustering. We present the approximation guarantees available for each and discuss practical merits and limitations. Surprisingly, despite the breadth of the literature explored, we conclude that there is still a gap between theory and practice: the most scalable methods are only intuitively motivated or loosely controlled, whereas those that come with end-to-end guarantees rely on strong assumptions or enable a limited gain of computation time.
	\end{abstract}

	\section{Introduction}

	Clustering is a cornerstone of our learning process and, thus, of our understanding of the world. Indeed, we can all distinguish between a rose and a tulip precisely because we have learned what these flowers \emph{are}.
	Plato would say that we learned the Idea --or Form~\cite{sedley_introduction_2016}-- of both the rose and the tulip, which then enables us to recognize all instances of such flowers. A machine learner would say that we learned two \emph{classes}: their most discriminating features (shape, size, number of petals, smell, etc.) as well as their possible intra-class variability. 
	
	Mathematically speaking, the first step on the road to classifying objects (such as flowers) is to create an abstract representation of these objects: with each object $i$ we associate a feature vector $p_i\in\mathbb{R}^d$, where the dimension $d$ of the vector corresponds to the number of features one chooses to select for the classification task. The space $\mathbb{R}^d$ in this context is sometimes called the \emph{feature space}. The choice of representation will obviously have a strong impact on the subsequent classification performance. 
	Say that in the flower example we choose to represent each flower by only $d=3$ features: the average color of each RGB channel (level of red, green and blue) of its petals. This choice is not fail-proof: even though the archetype of the rose is red and the archetype of the tulip is yellow, we know that some varieties of both flowers can have very similar colors and thus a classification solely based on the color will necessarily lead to confusion. In fact, there are many different ways of choosing features: from features based on the expert knowledge of a botanist, to features learned by a deep learning architecture from many instances of labeled images of roses and tulips, via features obtained by hybrid methods more-or-less based on human intelligence (such as the first few components of a Principal Component Analysis of expert-based features). 
	
	The second step on the road to classifying $n$ objects is to choose a machine learning algorithm that groups the set of $n$ points $\sf{P}=(p_1, \ldots, p_n)$ in $k$ classes ($k$ may be known in advance or determined by the algorithm itself). Choosing an appropriate algorithm depends on the context:
	\begin{itemize}
		\item \textbf{Availability of pre-labeled data.} Classifying the points $\sf{P}$ in $k$ classes may be seen as assigning a label (such as ``rose" or "tulip" in our $k=2$ example) to each of the points. If one has access to some pre-labeled data, we are in the case of \emph{supervised learning}: a more-or-less parametrized model is first learned from the pre-labeled data and then applied to the unlabeled points that need classification. If one does not have access to any pre-labeled data, we are in the case of \emph{unsupervised learning} where classes are typically inferred only via geometrical consideration of the distribution of points in the feature space. If one has only access to a few labeled data, we are in the in-between case of \emph{semi-supervised learning} where the known labels are typically propagated in one form or another in the feature space. 
		\item \textbf{Inductive vs transductive learning.} Another important characteristic of a classification algorithm is whether it can be used to classify only the set of points $\sf{P}$ at hand (transductive), or if it can also be directly used to classify any never-seen data point $p_{n+1}$ (inductive).
	\end{itemize}
	This review focuses on the family of algorithms jointly referred to as \emph{spectral clustering}. These algorithms are unsupervised and transductive: no label is known in advance and one may not naturally\footnote{Out-of-sample extensions of spectral clustering do exist (see for instance Section 5.3.6 of~\cite{wierzchon_spectral_2018}), but they require additional work. } extend the results obtained on $\sf{P}$ to never-seen data points. Another particularity of spectral clustering algorithms is that the number of classes $k$ is known in advance. 
	
	Spectral clustering algorithms have received a large attention in the last two decades due to their good performance on a wide range of different datasets, as well as their ease of implementation. In a nutshell, they combine three steps:
	\begin{enumerate}
		\item  \textbf{Graph construction.} A sparse similarity graph is built between the $n$ points. 
		\item  \textbf{Spectral embedding.} The first $k$ eigenvectors of a graph representative matrix (such as the Laplacian) are computed.
		\item \textbf{Clustering.} $k$-means is performed on these spectral features, to obtain $k$ clusters.  
	\end{enumerate}
	For background information about spectral clustering, such as several justifications of its performance, out-of-sample extensions, as well as comparisons with local methods, the interested reader is referred to the recent book chapter~\cite{wierzchon_spectral_2018}.
	
	One of the drawbacks of spectral clustering is its computational cost as $n$, $d$, and/or $k$ become large (see Section~\ref{subsec:computational_complexity} for a discussion on the cost). Since the turn of the century, a large number of authors have striven to reduce the computational cost while keeping the high level of classification performance. The majority of such accelerating methods are based on sampling: they reduce the dimension of a sub-problem of spectral clustering, compute a low-cost solution in small dimension, and lift the result back to the original space.

	The goal of this paper is to review existing sampling methods for spectral clustering, focusing especially on their approximation guarantees. Some of the fundamental questions we are interested in are: \emph{where is the sampling performed and what is sampled precisely? how should the reduced approximate solutions be lifted back to the original space? what is the computational gain? what is the control on performances---if it exists?} Given the breadth of the literature on the subject, we do not try to be exhaustive, but rather to illustrate the key ways that sampling can be used to provide acceleration, paying special attention on recent developments on the subject. \\  
	
	
	\noindent \textbf{Paper organization.} We begin by recalling in Section~\ref{sec:SC} the prototypical spectral clustering algorithm. We also provide some intuitive and formal justification of why it works. The next three sections classify the different methods of the literature depending on where the sampling is performed with respect to the three steps of spectral clustering: 
	\begin{itemize}
		\item Section~\ref{sec:from_scratch} details methods that sample  directly in the original feature space.  
		\item Section~\ref{sec:spectral_embedding} assumes that the similarity graph is given and details methods that sample nodes and/or edges to approximate the spectral embedding.
		\item Section~\ref{sec:kmeans} assumes that the spectral embedding is given and details methods to accelerate the $k$-means step. 
	\end{itemize}
	Finally, Section~\ref{sec::perspectives} gives perspective on the limitations of existing works and discusses {key open problems}. \\ 
	
	\noindent \textbf{Notation.} Scalars, such as $\lambda$ or $d$, are written with low-case letters. Vectors, such as $u$, $z$ or the all-one vector $\mathbf{1}$, are denoted by low-case bold letters. Matrices, such as $W$ or $L$ are denoted with bold capital letters. Ensembles are denoted by serif font capital letters, such as $\sf{C}$ or $\mathsf{X}$. The ``tilde'' will denote approximations, such as in $\tilde{z}$ or $\tilde{U}_k$. We use so-called Matlab notations to slice matrices: given a set of indices $\sf{S}$ of size $m$ and an $n\times n$ matrix $W$, $W(\sf{S},:)\in\mathbb{R}^{m\times n}$ is $W$ reduced to the lines indexed by $\sf{S}$, $W(:,\sf{S})\in\mathbb{R}^{n\times m}$ is $W$ reduced to the columns indexed by $\sf{S}$, and $W(\sf{S},\sf{S})\in\mathbb{R}^{m\times m}$ is $W$ reduced to the lines and columns indexed by $\sf{S}$. The equation $U_k=U(:,:k)$ defines $U_k$ as the reduction of $U$ to its first $k$ columns. Also, $\bf{C}^\top$ is the transpose of matrix $\bf{C}$ and  $\bf{C}^+$ its Moore-Penrose pseudo-inverse. The operator $X = \text{diag}(x)$ takes as an input a vector $x \in \mathbb{R}^n$ and returns an $n\times n$ diagonal matrix $X$ featuring $x$ in its main diagonal, i.e., $X(i,j) = x(i)$ if $i=j$ and $X(i,j) = 0$, otherwise. Finally, we will consider graphs in a large part of this paper. We will denote by $G=(\sf{V}, \sf{E}, W)$ the undirected weighted graph of $|\sf{V}|=n$ nodes interconnected by $|\sf{E}|=e$ edges: $e_{ij}\in\sf{E}$ is the edge connecting nodes $i$ and $j$, with weight $W(i,j) \geq 0$. Matrix $W$ is the adjacency matrix of $G$. As $G$ is considered undirected, $W$ is also symmetric.

	
	\section{Spectral clustering}
	\label{sec:SC}

	The input of spectral clustering algorithms consists of (i) a set of points $\sf{P}=(p_1, p_2, \ldots, p_n)$ (also called featured vectors) representing $n$ objects in a feature space of dimension $d$, and (ii) the number of classes $k$ in which to classify these objects. The output is a partition of the $n$ objects in $k$ disjoint clusters.
	The prototypical spectral clustering algorithm~\cite{shi_normalized_2000, ng_spectral_2002}, dates back in fact to fundamental ideas by Fiedler~\cite{fiedler_algebraic_1973} and entails the following steps:

	\mybox{
	\textbf{Algorithm 1. The prototypical Spectral Clustering algorithm}\\
	\noindent\textbf{Input:} A set of $n$ points $\sf{P}=(p_1, p_2, \ldots, p_n)$ in dimension $d$ and a number of desired clusters  $k$.
	\begin{enumerate}[itemsep=-0.2ex]
		\item \text{Graph construction} (optional)
		\begin{enumerate}[topsep=-1cm,itemsep=-0.2ex]
			\item Compute the kernel matrix $K\in\mathbb{R}^{n\times n}$: $~~\forall(i,j),~ K(i,j) =  \kappa(\|p_i - p_j\|_2)$.
			\item Compute $W = s(K)$, a sparsified version of $K$.  
			\item Interpret $W$ as the adjacency matrix of a weighted undirected graph $G$. 
		\end{enumerate}
		\item \text{Spectral embedding} 
		\begin{enumerate}[topsep=-1cm,itemsep=-0.2ex]
			\item Compute the eigenvectors $u_1,\:u_2,\:\cdots,\:u_{k}$ associated with the $k$ smallest eigenvalues of a graph  representative matrix $R$ of $G$.
			\item Set $U_k = [ \:u_1|\:u_2|\:\cdots|\:u_{k}\: ]\in\mathbb{R}^{n\times k}$.
			\item Embed the $i$-th node to $x_i = \frac{U_k(i,:)^\top}{q(\|U_k(i,:)\|_2)}$, with $q(\cdot)$ a normalizing function. 
		\end{enumerate}
		\item \text{Clustering} 
		\begin{enumerate}[topsep=-1cm,itemsep=-0.2ex]
			\item Use $k$-means on $x_1, \ldots, x_n$ in order to identify $k$ centroids $c_1, \ldots, c_k$.
			\item Voronoi tesselation: construct one cluster 
			per centroid $c_\ell$ and assign each object $i$ to the cluster of the centroid closest to $x_i$.
		\end{enumerate}
	\end{enumerate}
	\noindent\textbf{Output:} A partition of the $n$ points in $k$ clusters.
}
	
	\noindent A few comments are in order:
	\begin{itemize}
		\item A common choice of kernel in step 1a is the radial basis function (RBF) kernel $\kappa(\|p_i - p_j \|_2) = \exp\left(-\|p_i - p_j \|_2^2/\sigma^2\right)$ for some user-defined $\sigma$. 
		The sparsification $s$ of $K$ usually entails setting the diagonal to $0$ and keeping only the $k$ largest entries of each column (i.e., set all others to $0$). The obtained matrix $K_{\text{sp}}$ is not symmetric in general and a final ``symmetrization'' step $W=K_{\text{sp}}+K_{\text{sp}}^\top$ is necessary to obtain a matrix $W$ interpretable as the adjacency matrix of a weighted undirected graph\footnote{Each node $i$ of $\sf{V}$ represents a point $p_i$, an undirected edge exists between nodes $i$ and $j$ if and only if $W(i,j)\neq0$, and the weight of that connection is $W(i,j)$.} $G=(\sf{V},\sf{E}, W)$.  This graph is called the $k$ nearest neighbour ($k$-NN) similarity graph (note that the $k$ used in this paragraph has nothing to do with the number of clusters). Other kernel functions $\kappa$ and sparsification methods are possible (see Section 2 of~\cite{von2007tutorial} for instance). 
		\item There are several possibilities for choosing the graph representative matrix $R$ in step 2a. We consider three main choices~\cite{von2007tutorial}. Let us denote by $D$ the diagonal degree matrix such that $D(i,i)=\sum_j W(i,j)$ is the (weighted) degree of node $i$. We define the \emph{combinatorial} graph Laplacian matrix $L = D-W$, the \emph{normalized} graph Laplacian matrix $L_n = I - D^{-1/2} W D^{-1/2}$, and the \emph{random walk} Laplacian $L_{rw} = I - D^{-1}W$. Other popular choices include\footnote{In some of these examples, the $k$ largest eigenvalues (instead of the $k$ lowest in the Laplacian cases) of the representative matrix, and especially their corresponding eigenvectors, are of interest. This is only a matter of sign of the matrix $R$ and has no impact on the general discussion.} the non-backtracking matrix~\cite{krzakala_spectral_2013}, degree-corrected versions of the modularity matrix~\cite{ali_improved_nodate}, the Bethe-Hessian matrix~\cite{saade_spectral_nodate} or similar deformed Laplacians~\cite{dallamico_optimized_2019}. 
		\item The normalizing function $q(\cdot)$ depends on which  representative matrix is chosen. In the case of the Laplacians, experimental evidence as well as some theoretical arguments~\cite{von2007tutorial} support using a unit norm normalization for the eigenvectors of $L_n$ (i.e. $q$ is the identity function), and no normalization for the eigenvectors of $L$ and $L_{rw}$ (i.e. $q(\cdot)=1$). 
		\item Step 1 of the algorithm is ``optional'' in the sense that in some cases the input is not a set of points but directly a graph. For instance, it could be a graph representing a social network between $n$ individuals, where each node is an individual and there is an edge between two nodes if they know each other. The weight on each edge can represent the strength of their relation (for instance close to 0 if they barely know each other, and close to 1 if they are best friends). The goal is then to classify individuals based on the structure of these social connections and is usually referred to as \emph{community detection} in this context~\cite{fortunato_community_2010}. Given the input graph, and the number $k$ of communities to identify, one can run spectral algorithms starting directly at step 2. Readers only interested in such applications can skip Section~\ref{sec:from_scratch}, which is devoted to sampling techniques designed to accelerate step 1. 
	\end{itemize}
	
	After the spectral embedding $\sf{X}=(\:x_1,\:\ldots,\:x_n\:)$ has been identified, spectral clustering uses $k$-means in order to find the set of $k$ centroids $\sf{C} = (\:c_1,\:\ldots,\:c_k\:)$ that best represents the data. Formally, the $k$-means cost function to minimize reads:
	\begin{align}
	\label{eq:def_k_means_cost}
	f(\sf{C}; \sf{X}) = \sum_{x\in\sf{X}}\:\: \min_{c\in\sf{C}} \|x - c\|^2.  
	\end{align}
	We would ideally hope to identify the set of $k$ centroids $\sf{C}^*$ 
    minimizing $f(\sf{C}; \sf{X})$. Solving exactly this problem is NP-hard~\cite{drineas_clustering_1999}, so one commonly resorts to approximation and heuristic solutions (see for instance \cite{tarsitano_computational_2003} for details on different such  heuristics). The most famous is the so-called Lloyd-Max heuristic algorithm:
	\mybox{
		\textbf{Algorithm 2. The Lloyd-Max algorithm~\cite{lloyd_least_1982}}\\
		\noindent\textbf{Input:} A set of $n$ points $\sf{X}=(x_1, x_2, \ldots, x_n)$ and a number of desired clusters  $k$.
		\begin{enumerate}[itemsep=-0.2ex]
			\item Start from an initial guess $\sf{C}_{\text{ini}}$ of $k$ centroids
			\item Iterate until convergence:
			\begin{enumerate}[topsep=-1cm,itemsep=-0.2ex]
				\item Assign each point $x_i$ to its closest centroid to obtain a partition of $\sf{X}$ in $k$ components.
				\item Update each centroid $c_\ell$ as the average position of all points $x$ in component $\ell$.
			\end{enumerate}
		\end{enumerate}
		\noindent\textbf{Output:} A set of $k$ centroids $\sf{C} = (c_1 , \ldots, c_k)$.
	}
	When the clusters are sufficiently separated and $\sf{C}_{ini}$ is not too far from the optimal centroids, then the Lloyd-Max algorithm converges to the correct solution~\cite{kumar2010clustering}. Otherwise, it typically ends up in a local minimum. \\
	
	%
	
	\noindent\textbf{A remark on notation.}  Two quantities of fundamental importance in spectral clustering are the eigenvalues $\lambda_i$ and especially the eigenvectors $u_i$ of the graph Laplacian matrix. We adopt the graph theoretic convention of sorting eigenvalues in non-decreasing order: $0=\lambda_1 \leq \lambda_2 \leq \ldots \leq  \lambda_n.$ Also, for reasons of brevity, we overload notation and use the same symbol for the spectrum of the three Laplacians $L$, $L_n$ and $L_{rw}$. Thus, we advise the reader to rely on the context in order to discern which Laplacian gives rise to the eigenvalues and eigenvectors. Finally, the reader should keep in mind that the largest eigenvalue if always bounded by $2$ for $L_n$ and $L_{rw}$. 
	
	\subsection{An illustration of spectral clustering}
	
	The first two steps of the algorithm can be understood as a non-linear transformation from the initial feature space to another feature space (that we call spectral feature space or spectral embedding): a transformation of features $p_i$ in $\Rbb^d$ to spectral features $x_i$ in $\Rbb^k$.
	The first natural question that arises is why do we run $k$-means on the spectral features $\sf{X}=(x_1, \ldots, x_n)$ that are subject to parameter tuning and costly to compute, rather than directly run $k$-means on the original $\sf{P}$? Figures~\ref{fig:kmeans_moons} and~\ref{fig:SC_moons} illustrate the answer. 
	
	In Figure~\ref{fig:kmeans_moons}, we show the result of $k$-means directly on a set of artificial features $\sf{P}$ known as the two-half moons dataset. In this example, the intuitive ground truth is that each half-moon corresponds to a class, that we want to recover. Running $k$-means directly in this $2D$ feature space will necessarily output a linear separation between the two obtained Voronoi cells and will thus necessarily fail, as no straight line can separate the two  half-moons. 
	
	\begin{figure}
		\begin{center}
			\includegraphics[width=0.55\textwidth]{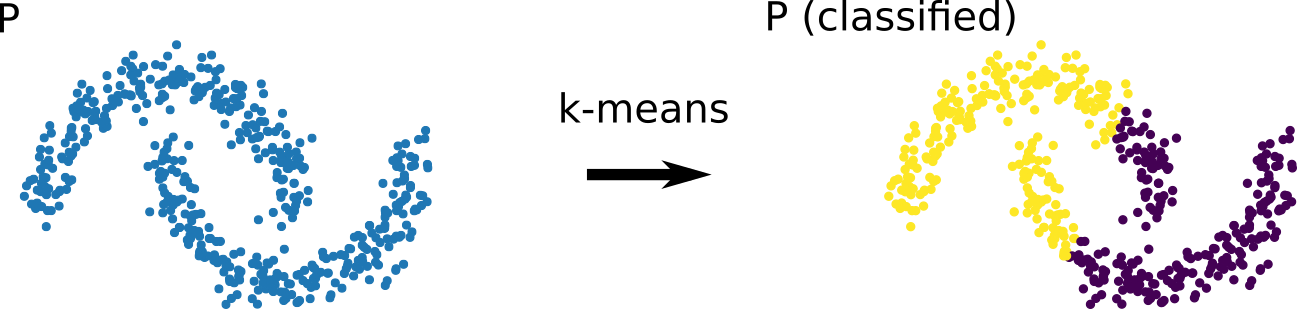}
		\end{center}
		\caption{ Left: the two half-moons synthetic dataset ($n=500$, $d=2$, $k=2$). Right: $k$-means with $k=2$ directly on $\sf{P}$ is unsuccessful to separate the two half-moons.}
		\label{fig:kmeans_moons}
	\end{figure}
	
	Spectral clustering, via the computation of the spectral features of a similarity graph, transforms these original features $\sf{P}$ in spectral features $\sf{X}$ that are typically linearly separable by $k$-means: the two half-moons are succesfully recovered! We illustrate this in Figure~\ref{fig:SC_moons}. To understand why this works, there are several theoretical arguments of varying rigour. We propose a few in the following. 
	
	\begin{figure}
		\begin{center}
		\includegraphics[width=0.85\textwidth]{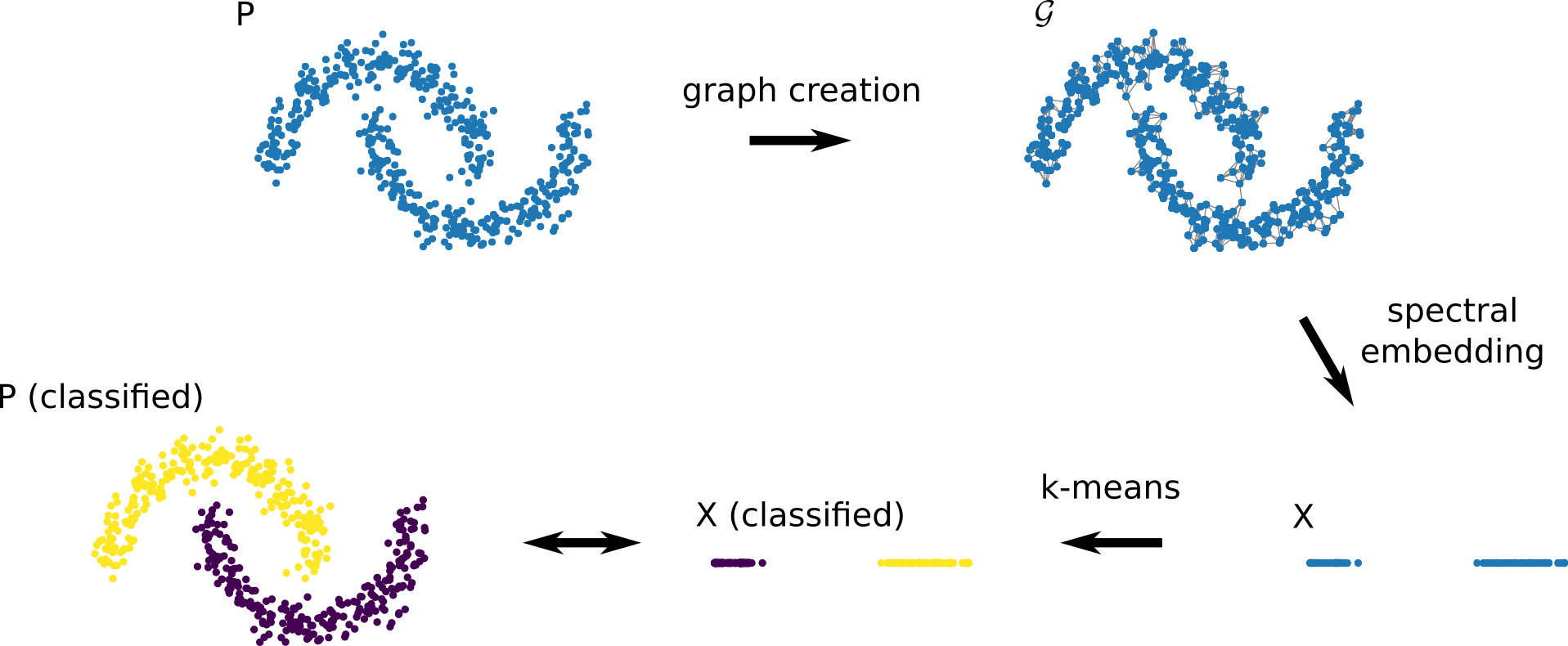}
	\end{center}
		\caption{ Illustration of the spectral clustering algorithm on the two half-moons dataset ($n=500$, $d=2$, $k=2$). The graph is created with a RBF kernel and via a sparsification done with $k$-nearest neighbours (with $k=5$). The spectral embedding is done with the two eigenvectors associated to the two smallest eigenvalues of the combinatorial Laplacian matrix $L$. The embedding $\sf{X}$ is here in practice in $1D$ as the first eigenvector of $L$ is always constant and thus not discriminative. Observe how the two clusters are now linearly separable in the spectral feature space. $k$-means on \emph{these} features successfully recovers the two half-moons. }
		\label{fig:SC_moons}
	\end{figure}

	\subsection{Justification of spectral clustering}
	\label{subsec:graph_cut}
	A popular approach --and by no means the only one, see Section~\ref{subsubsec:relax_bro}-- to justify spectral clustering algorithms stems from its connection to graph partitioning. 
	Suppose that the similarity graph $G=(\sf{V}, \sf{E}, W)$ has been obtained and we want to compute a partition\footnote{By definition, a \emph{partition} $\mathcal{P}=\{\sf{V}_1, \sf{V}_2, \ldots, \sf{V}_k\}$ of the nodes $\sf{V}$ is such that  $\cup_{\ell=1,\ldots,k}\sf{V}_\ell=\sf{V}$ and $\forall \ell\neq \ell', \sf{V}_\ell\cap\sf{V}_{\ell'}=\emptyset$} $\mathcal{P}=\{\sf{V}_1, \sf{V}_2, \ldots, \sf{V}_k\}$ of the nodes $\sf{V}$ in $k$ groups. Intuitively, a good clustering objective function should favor strongly connected nodes to end up in the same subset, and nodes that are far apart in the graph to end up in different subsets. This intuition can be formalized with graph cuts.
	
	Considering two groups $\sf{V}_1$ and $\sf{V}_2$, define $w(\sf{V}_1,\sf{V}_2)=\sum_{i\in \sf{V}_1}\sum_{j\in \sf{V}_2} W(i,j)$ to be the total weight of all links connecting  $\sf{V}_1$ to  $\sf{V}_2$. Also, denote by $\bar{\sf{V}}_\ell$ the complement of $\sf{V}_\ell$ in $\sf{V}$, such that $w(\sf{V}_\ell,\bar{\sf{V}}_\ell)$ is the total weight one needs to cut in order to disconnect $\sf{V}_\ell$ from the rest of the graph. 
	Given these definitions, the simplest graph cut objective function, denoted by $\texttt{cut}$, is:
	\begin{align}
	\texttt{cut}\left(\mathcal{P}=\{\sf{V}_1, \ldots, \sf{V}_k\}\right)	 = \frac{1}{2}\sum_{\ell=1}^k w(\sf{V}_\ell, \bar{\sf{V}}_\ell).
	\end{align}
	The best partition according to the $\texttt{cut}$ criterion is 
	$\mathcal{P}^* = \text{argmin}_{\mathcal{P}} ~~\texttt{cut}(\mathcal{P})$. 
	For $k=2$, solving this problem can be done exactly in $O(n e + n^2 \log(n))$\ amortized time using the Stoer-Wagner algorithm~\cite{stoer1997simple} and approximated in nearly linear time~\cite{karger2000minimum}. Nevertheless, this criterion is not satisfactory as it often separates an individual node from the rest of the graph, with no attention to the balance of the sizes or volumes of the groups. In clustering, one usually wants to partition into groups that are ``large enough". There are two famous ways to balance the previous cost in the machine learning literature\footnote{The reader should note that in the graph theory literature, the measure of conductance is preferred over \ncut. Conductance is $\max_{\ell}  w(\sf{V}_\ell, \bar{\sf{V}}_\ell) / w(\sf{V}_\ell) $. The two measures are equivalent when $k=2$.}: the \emph{ratio cut}~\cite{wei_towards_nodate} and  \emph{normalized cut}~\cite{shi_normalized_2000} cost functions, respectively defined as: 
	\begin{align}
	\rcut(\mathcal{P})	 = \frac{1}{2}\sum_{\ell=1}^k \frac{w(\sf{V}_\ell, \bar{\sf{V}}_\ell)}{|\sf{V}_\ell|} ~~~~\text{and}~~~~\ncut(\mathcal{P})	 = \frac{1}{2}\sum_{\ell=1}^k \frac{w(\sf{V}_\ell, \bar{\sf{V}}_\ell)}{\text{vol}(\sf{V}_\ell)},
	\end{align}
	where $|\sf{V}_\ell|$ is the number of nodes in $\sf{V}_\ell$ and 
	$\text{vol}(\sf{V}_\ell)=\sum_{i\in\sf{V}_\ell} \sum_{j\in\sf{V}} W(i,j)$ is the so-called volume of $\sf{V}_\ell$. 
	The difference between them is that \ncut\ favors clusters of large volume, whereas \rcut\ only considers cluster size---though for a $d$-regular graph with unit weights the two measures match (up to multiplication by $1/d$). 
	Unfortunately, it is hard to minimize these cost functions directly: minimizing these two balanced costs is NP-hard~\cite{wagner_between_1993, shi_normalized_2000} and one needs to search over the space of all possible partitions which is of exponential size. \\
	
	\noindent\textbf{A continuous relaxation.} Spectral clustering may be interpreted as a continuous relaxation of the above minimization problems. Without loss of generality, in the following we concentrate on relaxing the \rcut\ minimization problem (\ncut\ is relaxed almost identically). 
	Given a partition $\mathcal{P}=(\sf{V}_1, \ldots, \sf{V}_k)$, let us define
	\begin{align}
	\label{eq:X}
	{\bf{C}} = \left(\frac{z_1}{\sqrt{|\sf{V}_1|}}|\ldots|\frac{z_k}{\sqrt{|\sf{V}_k|}}\right) \in\mathbb{R}^{n\times k},
	\end{align}
	where $z_\ell\in\mathbb{R}^n$ is the indicator vector of $\sf{V}_\ell$:
	\begin{align}
	z_\ell(i) &= 
	\begin{cases}
	1 & \text{if node } i\in\sf{V}_\ell, \\
	0 & \text{otherwise.}
	\end{cases}
	\end{align}
	It will prove useful in the following to remark that, independently of how the partitions are chosen, we always have that ${\bf{C}}^\top {\bf{C}}=I$, the identity matrix in dimension $k$. With this in place, the problem of minimizing \rcut\ can be rewritten as (see discussion in~\cite{von2007tutorial}):
	\begin{align}
	\min_{{\bf{C}} \in \Rbb^{n\times k}} \trace {{\bf{C}}^\top L {\bf{C}}}   
	\quad \text{s.t.} \quad {\bf{C}}^\top {\bf{C}} = I \quad \text{and $~~~{\bf{C}}$ as in Eq.~\eqref{eq:X}}
	\label{eq:rcut_opt}
	\end{align}
	%
	%
	%
	To understand why this equivalence holds, one should simply note that 
	\begin{align} 
	\trace {{\bf{C}}^\top L {\bf{C}}} = \sum_{\ell=1}^k \frac{1}{|\sf{V}_\ell|} z_\ell^\top L z_\ell
	&= \sum_{\ell=1}^k \frac{1}{|\sf{V}_\ell|}\sum_{i>j} W(i,j) (z_\ell(i)-  z_\ell(j))^2 \notag \\ 
	&= \sum_{\ell=1}^k \frac{w(\sf{V}_\ell, \bar{\sf{V}}_{\ell})}{|\sf{V}_\ell|} = 2 \;\rcut(\mathcal{P}). \notag 
	\end{align}
	%
	%
	%
	Solving Eq.~\eqref{eq:rcut_opt} is obviously still NP-hard as the only thing we have achieved is to rewrite the \rcut\ minimization problem in matrix form. Yet, in this form, it is easier to realize that one may find an approximate solution by relaxing the discreteness constraint ``$\bf{C}$ as in Eq.~\eqref{eq:X}''. In the absence of the hard-to-deal-with constraint, the relaxed problem is not only polynomially solvable but also possesses a closed-form solution! 
	By the Courant–Fischer–Weyl (min-max) theorem, the solution is given by the first $k$ eigenvectors $U_k = [u_1,u_2, \ldots, u_k]$ of $L$:
	\begin{align}
	U_k = \argmin_{{\bf{C}} \in \Rbb^{n\times k}} \;\trace{{\bf{C}}^\top L {\bf{C}}} \quad \text{subject to} \quad {\bf{C}}^\top {\bf{C}} = I. \notag
	\end{align}
	This relaxation is not unique to the combinatorial Laplacian. In the same spirit, the minimum \ncut\ optimization problem can be formulated in terms of the  normalized Laplacian matrix $L_n$, and the relaxed problem's solution is given by the first $k$ eigenvectors of $L_n$. 
	%
	%
	%
	
	A difficulty still lies before us: how do we go from a real-valued $U_k$ to a partition of the nodes? The two next subsections aim to motivate the use of $k$-means as a rounding heuristic. The exposition starts from the simple case when there are only two clusters ($k=2$) before considering the general case (arbitrary $k$). 
	
	\subsubsection{The case of two clusters: thresholding suffices}
	
	%
	For simplicity, we first consider the case of two clusters. If one constructs a partitioning $\mathcal{P}_t$ with $\sf{V}_1 = \{ v_i : u_2(i)>t\}$ and $\sf{V}_2 = \{ v_i : u_2(i) \leq t\}$ for every level set $t \in (-1,1)$, then it is a folklore result that 
	\begin{align}
	\rcut(\mathcal{P}^*) \leq \min_{t}\; \rcut(\mathcal{P}_t) \leq 2 \sqrt{ \rcut(\mathcal{P}^*) \, \left( d_{\text{max}} - \frac{\lambda_2}{2}\right)},
	\label{eq:rcut_guarantee}
	\end{align}
	with $\mathcal{P}^* = \argmin_\mathcal{P} \ \rcut(\mathcal{P})$ being the optimal partitioning, $d_{\text{max}}$ is the maximum degree of any node in $\sf{V}$, and $\lambda_2$ the second smallest eigenvalue of $L$. The upper bound is achieved by the tree-cross-path graph constructed by Guattery and Miller~\cite{guattery1998quality}. 
	In an analogous manner, if $\mathcal{P}^* = \argmin_\mathcal{P} \ \ncut(\mathcal{P})$ is the optimal partitioning w.r.t. the \ncut\ cost and every $\mathcal{P}_t$ has been constructed by thresholding the second eigenvector of $L_n$, then
	\begin{align}
	\ncut(\mathcal{P}^*) \leq \min_{t} \;\ncut(\mathcal{P}_t) \leq 2 \sqrt{ \ncut(\mathcal{P}^*)}. 
	\label{eq:ncut_guarantee}
	\end{align}
	Equation~\eqref{eq:ncut_guarantee} can be derived as a consequence of the \emph{Cheeger inequality}, a key result of spectral graph theory~\cite{chung1997spectral}, which for the normalized Laplacian reads:
	\begin{align}
	\frac{\lambda_2}{2} \leq \ncut(\mathcal{P}^*) \leq \min_{\sf{V}} \frac{w(\sf{V}, \bar{\sf{V}})}{ \min\{ w(\sf{V}), w(\bar{\sf{V}})\} }  \leq \min_{t} \;\ncut(\mathcal{P}_t) \leq \sqrt{ 2 \lambda_2 }. \notag
	\end{align}
    %
    %
    As a consequence, we have
	$$ \ncut(\mathcal{P}^*) \leq \min_{t} \;\ncut(\mathcal{P}_t) \leq \sqrt{ 2 \lambda_2 } \leq \sqrt{ 4 \ncut(\mathcal{P}^*)} =  2 \sqrt{ \ncut(\mathcal{P}^*)},$$
    as desired.
	The derivation of the \rcut\ bound given in equation~\eqref{eq:rcut_guarantee} follows similarly.

	\subsubsection{More than two clusters: use $k$-means} 
	
	As the number of clusters $k$ increases, the brute-force approach of testing every level set becomes quickly prohibitive. 
	But why is $k$-means the right way to obtain the clusters in the spectral embedding? Though a plethora of experimental evidence advocate the use of $k$-means, a rigorous justification is still lacking. The interested reader may refer to~\cite{lee2014multiway} for an example of an analysis of spectral partitioning without $k$-means.  
	
	More recently, Peng et al.~\cite{peng2015partitioning} came up with a mathematical argument showing that, if $G$ is well clusterable and we use a $k$-means algorithm (e.g., ~\cite{kumar2004simple}) which guarantees that the identified solution $\tilde{\sf{C}}$ abides to
	\begin{align}
	f(\tilde{\sf{C}}; \sf{X}) \leq (1 + \epsilon) f(\sf{C}^*; \sf{X}),  \notag  
	\end{align}
	where $\sf{C}^*$ is the optimal solution of the $k$-means problem, then the partitioning $\tilde{\mathcal{P}}$ produced by spectral clustering when using $L_{n}$ has \ncut\ cost provably close to that of the optimal partitioning $\mathcal{P}^*$. In particular, it was shown that, as long as $\lambda_{k+1} \geq \text{c} k^2 \ncut(\mathcal{P}^*)$, then 
    %
	%
	\begin{align}
	\ncut(\mathcal{P}^*) \leq  \ncut(\tilde{\mathcal{P}}) \leq \zeta \, \ncut(\mathcal{P}^*)  \left(1 + \epsilon \, \frac{k^3}{\lambda_{k+1}}\right), \notag
	\end{align}
	for some constants $\text{c},\zeta>0$ that are independent of $n$ and $k$ (see also~\cite{kolev2015note}). Note that, using the  higher-order Cheeger inequality~\cite{lee2014multiway} $\lambda_k/2 \leq \ncut(\mathcal{P}^*)$, the condition $\lambda_{k+1} \geq \text{c} k^2 \ncut(\mathcal{P}^*)$ implies 
	$$ \frac{\lambda_{k+1}}{\lambda_k} \geq \frac{\text{c} k^2}{2} = \Omega(k^2).$$ 

	Though hopefully milder than this one\footnote{To construct an example possibly verifying such a strong gap assumption, consider $k$ cliques of size $k$ connected together via only $k-1$ edges, so as to form a loosely connected chain. Even though this is a straightforward clustering problem known to be easy for spectral clustering algorithms, the above theorem's assumption implies
	$\lambda_{k+1} = \Omega( k^2 \ncut(\mathcal{P}^*)) =  \Omega(k)$ which, independently of $n$, can only be satisfied when $k$ is a small (recall that the eigenvalues of $L_n$ are necessarily between $0$ and $2$).},
	such gap assumptions are very common in the analysis of spectral clustering. Simply put, the larger the gap $\lambda_{k+1} - \lambda_k$ is, the stronger the cluster structure and the easier it is to identify a good clustering. Besides quantifying the difficulty of the clustering problem, the gap also encodes the robustness of the spectral embedding to errors induced by approximation algorithms~\cite{davis1963rotation}. The eigenvectors of a \emph{perturbed} Hermitian matrix exhibit an interesting property: instead of changing arbitrarily, the leakage of information is localized w.r.t. the eigenvalue axis~\cite{pmlr-v70-loukas17a}. More precisely, if $\tilde{u}_i$ is the $i$-th eigenvector of $L$ after perturbation, then the inner products $(\tilde{u}_i^\top u_j)^2$ decrease proportionally with $|\lambda_i - \lambda_j|^2$. As such, demanding that $\lambda_{k+1} - \lambda_k$ is large is often helpful in the analysis of spectral clustering algorithms in order to ensure that the majority of useful information (contained within $U_k$) is preserved (in $\tilde{U}_k$) despite approximation errors\footnote{Usually, one needs to ensure that $\sum_{i\leq k, j>k} (\tilde{u}_i^\top u_j)^2 / k$ remains bounded.}. 
	
	
	\subsubsection{Choice of relaxation}
	\label{subsubsec:relax_bro}
	The presented relaxation approach is not unique and other relaxation could be equally valid (see for instance~\cite{bie2006fast, bresson_multiclass_2013, rangapuram_tight_2014}). 
	This relaxation has nevertheless the double advantage of being theoretically simple and computationally easy to implement. Also, justification of spectral clustering algorithms does not only come from this graph cut perspective and in fact encompasses several approaches that we will not detail here: perturbation approaches or hitting time considerations~\cite{von2007tutorial}, a polarization theorem~\cite{brand_unifying_2003},  consistency derivations~\cite{ulrike_von_luxburg_consistency_2008, lei_consistency_2015}, etc. Interestingly, recent studies (for instance \cite{bordenave_non-backtracking_2015-1}) on the Stochastic Block Models have shown that spectral clustering (on other matrices than the Laplacian, such as the non-backtracking matrix~\cite{krzakala_spectral_2013}, the Bethe-Hessian matrix~\cite{saade_spectral_nodate} or other similar deformed Laplacians~\cite{dallamico_optimized_2019}) 
	perform well up to the detectability threshold of the block structure.

	\subsection{Computational complexity considerations}
	\label{subsec:computational_complexity}
	
	What is the computational complexity of spectral clustering as a function of the number of points $n$, their dimension $d$ and the number of desired clusters $k$? Let us examine the three steps involved one by one. 
	
	The first step entails the construction of a sparse similarity graph from the input points, which is dominated by the kernel computation and costs $\O(dn^2)$. 
	%
	In the second step, given the graph $G$ consisting of $n$ nodes and $e$ edges\footnote{with $e$ of the order of $n$ if the sparsification step was well conducted}, one needs to compute the spectral embedding (step 2 of Algorithm 1).  Without exploiting the special structure of a graph Laplacian ---other than its sparsity that is--- there are two main options:
	\begin{itemize}
		\item Using power iterations, one may identify sequentially each non-trivial eigenvector $u_\ell$ in time $\O( e / \delta_{\ell})$, where $\delta_{\ell} = \lambda_\ell - \lambda_{\ell-1}$ is the $\ell$-th eigenvalue gap and $e$ is the number of edges of the graph~\cite{vishnoi2013lx}. Computing the spectral embedding therefore takes $\O(k e / \delta) $ with $\delta = \min_{\ell} \delta_\ell$. Unfortunately, there exist graphs\footnote{The combinatorial Laplacian of a complete balanced binary tree on $k \geq 3$ levels and $n = 2^k-1$ nodes has $\frac{1}{n} \leq \lambda_2 \leq \frac{2}{n}$~\cite{guattery1995performance}.} such that $\delta = \O(1/n)$, bringing the overall worst-case complexity to $\O(k n e)$.
		\item The Lanczos method can be used to approximate the first $k$ eigenvectors in roughly $\O(e k + nk^2)$ time. This procedure is often numerically unstable resulting to a loss of orthogonality in the computed Krylov subspace basis. The most common way to circumvent this problem is by implicit restart~\cite{calvetti1994implicitly}, whose computational complexity is not easily derived. The number of restarts, empirically, depend heavily on the eigenvalue distribution in the vicinity of $\lambda_k$: if $\lambda_k$ is in an eigenvalue bulk, the algorithms takes longer than when $\lambda_k$ is isolated. We decide to write the complexity of restarted Arnoldi as $\O(t(e k + nk^2))$ with $t$ modeling the number of restarts. Note that throughout this paper, $t$ will generically refer to a number of iterations in algorithm complexities. We refer the interested reader to~\cite{bai2000templates} for an in-depth perspective on Lanczos methods. 
	\end{itemize}
	The third step entails solving the $k$-means problem, typically by using the Lloyd-Max algorithm to converge to a local minimum of $f(\sf{C}; \sf{X})$. Since there is no guarantee that this procedure will find a good local minimum, it is usually rerun multiple times, starting in each case from randomly selected centroids $\sf{C}_{\text{ini}}$. 
	The computational complexity of this third step is $\O(t n k^2)$, where $t$ is a bound on the number of iterations required until convergence multiplied by the number of retries (typically around 10). 
	
	\subsection{A taxonomy of sampling methods for spectral clustering }
	For the remainder of the paper, we propose to classify sampling methods aiming at accelerating one or more of these three steps according to when  they sample. If they sample before step 1, they are detailed in Section~\ref{sec:from_scratch}. Methods that assume that the similarity graph is given or well-approximated and sample between steps 1 and 2 will be found in Section~\ref{sec:spectral_embedding}. Finally, methods that assume that the spectral embedding has been exactly computed or well-approximated and sample before the $k$-means step are explained in Section~\ref{sec:kmeans}. This classification of methods, like all classification systems, bears a few flaws. For instance, Nyström methods can be applied to both the context of Sections~\ref{sec:from_scratch} and~\ref{sec:spectral_embedding} and are thus mentioned in both. Also, we decided to include the pseudo-code of only a few chosen algorithms that we think are illustrative of the literature. This choice is of course subjective and debatable. Notwithstanding these inherent flaws, we hope that this classification clarifies the landscape of existing methods. 
	%
	
	\section{Sampling in the original feature space}
	\label{sec:from_scratch} 
	
	This section is devoted to methods that ambitiously aim to reduce the dimension of the spectral clustering problem even before the graph has been formed. Indeed, the naive way of building the similarity graph (step 1 of spectral clustering algorithms) costs $\O(dn^2)$ and, as such, is one of the the key bottlenecks of spectral clustering. 
    %
	It should be remarked that the present discussion fits into the wider realm of kernel approximation, a proper review of which cannot fit in this paper: we will thus concentrate on methods that were in practice used for spectral clustering. 
		
	\subsection{Nystr\"om-based methods}
	\label{subsec:Nystrom}
	
	The methods of this section aim to obtain an approximation $\tilde{U}_k$ of the exact spectral embedding $U_k$ via a sampling procedure in the original feature space. \\
	
	\noindent\textbf{The Nystr\"om method} is a well known algorithm for obtaining a rough low rank approximation of a positive semi-definite (PSD) matrix $A$. Here is a high level description of the steps entailed:
	\mybox{
		\textbf{Algorithm 3. Nystr\"om's method}
		\\
		\textbf{Input:} PSD matrix $A \in \Rbb^{n \times n}$, number of samples $m$,  desired rank $k$
		\begin{enumerate}[itemsep=-0.2ex]
			\item Let $\mathsf{S}$ be $m$ column indices chosen by some sampling procedure. 
			\item Denote by $B = A(\mathsf{S},\mathsf{S}) \in \Rbb^{m\times m}$ and $\mathbf{C} = A(:,\mathsf{S}) \in \Rbb^{n\times m}$ the sub-matrices indexed by $\mathsf{S}$.
			\item Let $B = \mathbf{Q} \mathbf{\Sigma} \mathbf{Q}^\top$ be the eigen-decomposition of $B$ with the diagonal of $\mathbf{\Sigma}$ sorted in decreasing magnitude. 
			\item Compute the rank-$k$ approximation of $B$ as $B_k = \mathbf{Q}_k \mathbf{\Sigma}_k \mathbf{Q}_k^\top$, where $\mathbf{Q}_k  = \mathbf{Q}(:,:k) \in \Rbb^{n\times k}$ and $\mathbf{\Sigma}_k = \mathbf{\Sigma}(:k,:k)$.  
		\end{enumerate}
		\textbf{Possible outputs:} \begin{itemize}[topsep=-1cm,itemsep=-0.2ex]
		    \item A low-rank approximation $\tilde{A}=\mathbf{C} B^+ \mathbf{C}^\top \in \Rbb^{n \times n}$ of $A$
		    \item A rank-$k$ approximation $\tilde{A}_k = \mathbf{C} B^+_k \mathbf{C}^\top \in \Rbb^{n \times n}$ of $A$
		    \item The top $k$ eigenvectors of $\tilde{A}_k$, stacked as columns in matrix $\tilde{{\bf{V}}}_k\in\mathbb{R}^{n\times k}$, obtained by orthonormalizing the columns of $\tilde{\mathbf{Q}}_k=\mathbf{C} \mathbf{Q}_k \mathbf{\Sigma}_k^{-1}\in\mathbb{R}^{n\times k}$
		\end{itemize} 
	}
	Various guarantees are known for the quality of $\tilde{A}$ depending on the type of sampling utilized and the preferred notion of error (spectral $\|.\|_2$ vs frobenius $\|.\|_F$ vs trace $\|.\|_*$ norm)~\cite{gittens2016revisiting,kumar2009sampling,frieze2004fast,zhang2008improved}. For instance:
	\begin{theorem}[Lemma 8 for $q=1$ in~\cite{gittens2016revisiting}] 
		Let $\epsilon\in (0,1)$ and $\delta\in (0,1)$. Consider $m$ columns drawn i.i.d. uniformly at random (with or without replacement). Then:
		$$ 
		\| A - \tilde{A}\|_2 \leq \left( 1 + \frac{n}{(1 - \epsilon) m} \right) \| A - A_k\|_2
		$$
		holds with probability at least $1-3\delta$, provided that $m\geq 2\epsilon^{-2}\mu k\log{(k/\delta)}$; 
		where $$\mu = \frac{n}{k} \max_{i = 1, \ldots, n} \|\mathbf{V}_k(i,:)\|^2_2$$ 
		is the coherence associated with the first $k$ eigenvectors $\mathbf{V}_k$ of $A$, and $A_k$ is the best rank-$k$ approximation of $A$.  
	\end{theorem}

	Guarantees independent of the coherence can be obtained for more advanced sampling methods. Perhaps the most well known method is that of leverage scores, where one draws $m$ samples independently by selecting (with replacement) the $i$-th column with probability $p_i = \|\mathbf{V}_k(i,:)\|_2^2/k$, called leverage scores.
	\begin{theorem}[Lemma 5 for $q=1$ in~\cite{gittens2016revisiting}] 
	\label{thm:lev_score_Nystrom}
		Let $\epsilon\in (0,1)$ and $\delta\in (0,1)$. Consider $m$ columns drawn i.i.d. with replacement from such a probability distribution. Then:
		$$ 
		\| A - \tilde{A}\|_2 \leq \| A - A_k\|_2 + \epsilon^2 \| A - A_k\|_*
		$$	
		holds 
		with probability at least $0.8 - 2 \delta$ provided that $m \geq \O(\epsilon^{-2}k \log(k/\delta))$.
	\end{theorem}
	Computing leverage scores exactly is computationally prohibitive since it necessitates a partial SVD decomposition of $A$, which we are trying to avoid in the first place. Nevertheless, it is possible to approximate all leverage scores with a multiplicative error guarantee in time roughly $\O(e k \log(e))$ if $A$ has $O(e)$ non-zero entries.
	(see Algorithms 1 to 3 in~\cite{gittens2016revisiting}). Many variants of the above exist~\cite{kumar2009sampling,kumar2012sampling}, but to the best of our knowledge, the fastest current Nystr\"om algorithm utilizes ridge leverage scores with a complex recursive sampling scheme and runs in time nearly linear in $n$~\cite{musco2017recursive}.\\
	
	\noindent\textbf{Nystr\"om for spectral clustering.}  
	Though initially conceived for low-rank approximation, Nystr\"om's method can also be used to accelerate spectral clustering. The key observation is that $U_k$, the tailing $k$ eigenvectors of the graph representative matrix $R$, can be interpreted as the top $k$ eigenvectors of the PSD matrix $A = \|R\|_2 I - R$. As such, the span of the $k$ top eigenvectors of $\tilde{A}_k$ obtained by running Algorithm 3 on $A$ is an approximation of the span of the exact spectral embedding. 
	Different variants of this idea have been considered for the acceleration of spectral clustering~\cite{fowlkes_spectral_2004, theeramunkong_approximate_2009,li2011time,bouneffouf2015sampling,mohan2017exploiting,li_towards_2017}.
	
	Following our taxonomy, we hereby focus on the case where we have at our disposal $n$ points $p_i$ in dimension $d$, and the similarity graph has yet to be formed. The case where the graph is known is deferred to Section~\ref{sec:spectral_embedding}. 
	
	
	%
	In this case, we cannot run Algorithm 3 on $A = \|R\|_2 I - R$ as the graph, and \emph{a fortiori} its representative matrix $R$, has not yet been formed. What we \emph{can} have access to \emph{efficiently} is $B=s(K(\sf{S},\sf{S}))$ and $\bf{C} = \textit{s}(K(:,\sf{S}))$ as these require only a partial computation of the kernel and cost only $\O(dnm)$. Note that $s$ is a sparsification function that is applied on a subset of the kernel matrix.
	The following pseudo-code exemplifies how Nyström-based techniques can be used to approximate the first $k$ eigenvectors $U_k$ associated with the normalized Laplacian matrix (i.e., here $R = L_n$):
\mybox{
		\textbf{Algorithm 3b. Nystr\"om for spectral clustering~\cite{li2011time}}\\
		\textbf{Input:} The set of points $\sf{P}$, the number of desired clusters $k$, a sampling set $\sf{S}$ of size $m\geq k$
		\begin{enumerate}[itemsep=-0.5ex]
			\item Compute the sub-matrices $B = s(K(\mathsf{S},\mathsf{S}))\in \Rbb^{m\times m}$ and ${\bf{C}} = s(K(:,\mathsf{S})) \in \Rbb^{n\times m}$, where $s$ is a sparsification function.
			\item Let $D_r = \text{diag}(B \mathbf{1})$ be the $m\times m$ degree matrix.
			\item Compute the top $k$ eigenvalues $\mathbf{\Sigma}_k$ and eigenvectors $\mathbf{Q}_k$ of $D_r^{-\sfrac{1}{2}} B D_r^{-\sfrac{1}{2}}$.
			\item Set $ \tilde{\mathbf{Q}}_k = \mathbf{C} D_r^{-\sfrac{1}{2}} \mathbf{Q}_k \mathbf{\Sigma}_k^{-1}$.
			\item Let $D_l = \text{diag}( \tilde{\mathbf{Q}}_k \mathbf{\Sigma}_k \tilde{\mathbf{Q}}_k^\top \mathbf{1})$ be the $n\times n$ degree matrix.
			\item Compute $\tilde{U}_k$ obtained by orthogonalizing $D_l^{-\sfrac{1}{2}} \tilde{\mathbf{Q}}_k$.
		\end{enumerate}
		\textbf{Output:} $\tilde{U}_k$, an approximation of the spectral embedding $U_k$.
}
	%
	This algorithm runs in $\O(nm \max(d,k))$ time, which is small when $m$ depends mildly on the other parameters of interest. 
	Nevertheless, the algorithm (and others like it) suffers from several issues:
	\begin{itemize}
	\item Alg. 3b attempts in fact to apply Nystr\"om on $A=2I - L_n=I+D^{-1/2} s(K) D^{-1/2}$, via the exact computation of two submatrices of $K$. It makes two strong (and uncontrolled) approximations. First of all, the sparsification step (step 1 in Alg. 3b) is  applied to the sub-matrices $K(\mathsf{S},\mathsf{S})$ and $K(:,\mathsf{S})$, deviating from the correct sparsification procedure that takes into account the entire kernel matrix $K$. Second, the degree matrix $D$ is never exactly computed as knowing it exactly would entail computing exactly $s(K)$, which is precisely what we are trying to avoid.  Existing methods thus rely on heuristic approximations of the degree in order to bypass this difficulty (see steps 2 and 5 of Alg. 3b). 
	\item Since we don't have direct access to the kernel matrix, we cannot utilize advanced sampling methods such as leverage scores to draw the sampling set $\sf{S}$. This is particularly problematic if (due to sparsification), matrices $B$ and $\bf{C}$ are sparse, as for sparse matrices uniform sampling is known to perform poorly~\cite{mohan2017exploiting}. Techniques that rely on distances between columns do not fair much better. Landmark-based approaches commonly perform better in simple problems but suffer when the clusters are non-convex~\cite{bouneffouf2015sampling}. We refer the reader to the work by Mohan et al.~\cite{mohan2017exploiting} for more information on landmark-based methods. The latter work also describes an involved sampling scheme that is aimed at general (i.e., non-convex) clusters.
	 
    \end{itemize}
	For the reasons highlighted above, the low-rank approximation guarantees accompanying the classical Nystr\"om method cannot be directly used here. $\emph{A fortiori}$, it is an open question how much the quality of the spectral clustering solution is affected by using the centroids obtained by running $k$-means on $\tilde{U}_k$. 
	%
	\\
	
	\noindent\textbf{Column sampling.} Akin in spirit to Nystr\"om methods, an alternative approach to accelerating spectral clustering was inspired by column sampling low-rank approximation techniques~\cite{drineas2006fast, deshpande2006matrix}. 
	
	An instance of such algorithms was put forth under the name of cSPEC (column sampling spectral clustering) by Wang et al.~\cite{theeramunkong_approximate_2009}. Let $\mathbf{C} = U_{C} \mathbf{\Sigma}_C \mathbf{V}_C^\top$ be the singular value decomposition of the $n\times m$ matrix $\mathbf{C}=s(K(:,\sf{S}))$. Then, matrices 
	$$ \tilde{\mathbf{\Sigma}} = \sqrt{ \frac{n}{m}} \, \mathbf{\Sigma}_C \quad \text{and} \quad \tilde{U} = \mathbf{C} \mathbf{V}_C \mathbf{\Sigma}_C^{+}$$
	are interpreted as an approximation of the actual eigenvalues and eigenvectors of $K$ and thus $U_k$ can be substituted by the first $k$ columns of $\tilde{U}$. This algorithm runs in $\O(ndm+nm^2)$. 

	Authors in~\cite{chen_large_2011} propose a hybrid method, between column sampling and the representative-based methods discussed in Section~\ref{subsec:rep_points}, where they propose the following approximate factorization of the data matrix:
	\begin{align}
	(p_1|\ldots|p_n) \simeq \mathbf{F} \mathbf{Z} \in\mathbb{R}^{d\times n},
	\end{align}
	where $\mathbf{F}\in\mathbb{R}^{d\times m}$ concatenates the feature vectors of $m$ sampled points and $\mathbf{Z}\in\mathbb{R}^{m\times n}$ represents all unsampled points as approximate linear combinations of the representatives, computed via sparse coding techniques~\cite{lee_efficient_2007}\footnote{Authors in~\cite{perner_fast_2009} have a very similar proposition as~\cite{chen_large_2011}, adding a projection phase at the beginning to reduce the dimension $d$ (see Section~\ref{subsubsec:feature_selection}). Similar ideas may also be found in~\cite{vladymyrov_fast_2017}.}. The SVD of $\tilde{\mathbf{D}}^{-1/2}\mathbf{Z}$, with $\tilde{\mathbf{D}}$ the row-sum of $\mathbf{Z}$, is then computed to obtain an approximation $\tilde{U}_k$ of $U_k$. The complexity of their algorithm is also $\O(ndm+nm^2)$. 
	
	In these methods, the choice of the sample set $\sf{S}$ is of course central and has been much debated. Popular options are uniformly at random or via better-taylored probability distributions, via a first $k$-means (with $k=m$) pass on $\sf{P}$, or via other selective sampling methods. 
	Also, as with most extensions of Nystr\"om's method to spectral clustering, column sampling methods for spectral clustering do not come with end-to-end approximation guarantees on $U_k$.
	
	In the world of low-rank matrix approximation the situation is somewhat more advanced. Recent work in column sampling utilizes adaptive sampling with leverage scores in time $\O(e + n \text{poly}(k))$, or uniformly i.i.d. after preconditioning by a fast randomized Hadamard transform~\cite{woodruff2014sketching,drineas2018lectures}. Others have also used a correlated version called volume sampling to obtain column indices~\cite{deshpande2006matrix}. Nevertheless, this literature extends beyond the scope of this review and we invite the interested reader to consider the aforementioned references for a more in depth perspective.

	\subsection{Random Fourier features}
	\label{subsec:RFF}
	Out of several sketching techniques one could \emph{a priori} use to accelerate spectral clustering, we focus on random Fourier features (RFF)~\cite{rahimi_random_2008}: a method that samples in the Fourier space associated to the original feature space. Even though RFFs have originally been developed to approximate a kernel matrix $K$ in time linear in $n$ instead of the quadratic time necessary for its exact computation, they can in fact be used to  obtain an approximation $\tilde{U}_k$ of the exact spectral embedding $U_k$. 
	
	Let us denote by $\kappa$ the RBF kernel, i.e.,  $\kappa(\mathbf{t})=\exp(-\mathbf{t}^2/\sigma^2)$, whose Fourier transform is:
	\begin{align}\label{eq:DefFourier}
	\hat{\kappa}(\mathbf{\omega})=\int_{\mathbb{R}^d}\kappa(\mathbf{t})\exp^{-i\mathbf{\omega}^\top \mathbf{t}}\mbox{d}\mathbf{t}. 
	\end{align}
	The above takes real values as 
	$\kappa$ is symmetric. One may write:
	\begin{equation}
	\kappa(p,\mathbf{q})=\kappa(p-\mathbf{q})=\frac{1}{Z}\int_{\mathbb{R}^d}\hat{\kappa}(\mathbf{\omega})\exp^{i\mathbf{\omega}^\top(p-\mathbf{q})}\mbox{d}\mathbf{\omega},
	\end{equation}
	where, in order to ensure that $\kappa(p, p)=1$, the normalization constant is set to $Z=\int_{\mathbb{R}^d}\hat{\kappa}(\mathbf{\omega})\mbox{d}\mathbf{\omega}$.
	According to Bochner's theorem, and due to the fact that $\kappa$ is positive-definite,  $\hat{\kappa}/Z$ is a valid
	probability density function. 
	$\kappa(p,\mathbf{q})$ may thus be interpreted as the expected value of $\exp^{i\mathbf{\omega}^\top(p-\mathbf{q})}$ 
	provided that $\mathbf{\omega}$ is drawn from $\hat{\kappa}/Z$:
	\begin{align}
	\kappa(p,\mathbf{q})=\mathbb{E}_{\mathbf{\omega}} \left(\exp^{i\mathbf{\omega}^\top(p-\mathbf{q})}\right)
	\end{align}
	Drawing $\omega$ from the distribution $\hat{\kappa}/Z$ is equivalent to drawing independently each of its $d$ entries  according to the normal law of mean $0$ and variance $2/\sigma^2$. Indeed: $\hat{\kappa}(\mathbf{\omega})=\pi^{d/2}\sigma^d\exp(-\sigma^2\mathbf{\omega}^2/4)$ and  $Z=\int_{\mathbb{R}^d}\hat{\kappa}(\mathbf{\omega})\mbox{d}\mathbf{\omega}=(2\pi)^d$, leading to $$\frac{\hat{\kappa}(\mathbf{\omega})}{Z}=\left(\frac{\sigma}{2\sqrt{\pi}}\right)^d\exp^{-\sigma^2\mathbf{\omega}^2/4}.$$ 
	In practice, we draw independently $m$ such vectors $\omega$ to obtain the set of sampled frequencies  $\Omega=(\mathbf{\omega}_1,\ldots,\mathbf{\omega}_m).$
	For each data point $p_i$, and given this set of samples $\Omega$, we define the associated random Fourier feature vector:
	\begin{align}
	\mathbf{\psi}_i=\frac{1}{\sqrt{m}} [\cos(\mathbf{\omega}_1^\top p_i)|\cdots|\cos(\mathbf{\omega}_m^\top p_i)|
	\sin(\mathbf{\omega}_1^\top p_i)|\cdots|\sin(\mathbf{\omega}_m^\top p_i)]^\top\in\mathbb{R}^{2m},
	\end{align}
	and call $\mathbf{\Psi}=\left(\mathbf{\psi}_1|\cdots|\mathbf{\psi}_n\right)\in\mathbb{R}^{2m\times n}$ the RFF matrix. Other embeddings are possible in the RFF framework, but this one was shown to be the most appropriate to the Gaussian kernel~\cite{sutherland_error_2015}. 
	As $m$ increases, $\mathbf{\psi}_i^\top\mathbf{\psi}_j$ concentrates around its expected value $\kappa(p_i,p_j)$:
	$\mathbf{\psi}_i^\top\mathbf{\psi}_j\simeq \kappa(p_i,p_j)$. Proposition 1 of~\cite{sutherland_error_2015} states the tightness of this concentration: it shows that the approximation starts to be valid with high probability for $m\geq \O(d\log{d})$. 
	The Gaussian kernel matrix is thus well approximated as $K\simeq \mathbf{\Psi}^\top\mathbf{\Psi}$. With such a low-rank approximation $\mathbf{\Psi}$ of $K$, one can: estimate the degrees\footnote{an approximation of the degree $d_i$ of node $i$ is $\psi_i^\top\bar{\psi}$ where $\bar{\psi}=\sum_{j}\psi_j$. All degrees can thus be estimated in time $\mathcal{O}(nm^2)$.}, degree-normalize $\mathbf{\Psi}$ to obtain a low-rank approximation of the normalized Laplacian $L_n$ and perform an SVD to directly obtain an approximation $\tilde{U}_k$ of the spectral embedding $U_k$. The total cost to obtain this approximation is $\mathcal{O}(ndm + nm^2)$. These ideas were developed in Refs.~\cite{chitta_efficient_2012, wu_scalable_2018} for instance.
	
	As in Nyström methods however, the concentration guarantees of RFFs for $K$ do not extend to the degree-normalized case; moreover, the sparsification step 1b of spectral clustering is ignored. Note that improving over RFFs in terms of efficiency and concentration properties is the subject of recent research (see for instance~\cite{le2013fastfood}). 
	
	\subsection{The paradigm of representative points}
	\label{subsec:rep_points}
	The methods detailed here sample in the original feature space and directly obtain a control on the misclustering rate due to the sampling process. They are based on the following framework:
	\begin{enumerate}
	    \item Sample $m$ so-called representatives. 
	    \item Run spectral clustering on the representatives. 
	    \item Lift the solution back to the entire dataset. 
	\end{enumerate}
	Let us illustrate this with the example of KASP:
	\mybox{
		\textbf{Algorithm 4. KASP: $k$-means-based approximate spectral clustering}~\cite{yan_fast_2009}\\
		\noindent\textbf{Input:} A set of $n$ points $\sf{P}=(p_1, p_2, \ldots, p_n)$ in dimension $d$, a number of desired clusters $k$, and a number of representatives $m$
		\begin{enumerate}[itemsep=-0.2ex]
			\item Perform $k$-means with $k=m$ on $\sf{P}$ and obtain:
			\begin{enumerate}[topsep=-1cm,itemsep=-0.2ex]
				\item the cluster centroids $\sf{Y}=(y_1 , \ldots, y_m)$ as the $m$ representative points.
				\item a correspondence table to associate each $p_i$ to its nearest representative
			\end{enumerate}
			\item Run spectral clustering on $\sf{Y}$: obtain a $k$-way cluster membership for each $y_i$.
			\item Lift the cluster membership to each $p_i$ by looking up the cluster membership of its   representative in the correspondence table.
		\end{enumerate}
		\noindent\textbf{Output:} $k$ clusters
	}
	The complexity of KASP is bounded by\footnote{It is in fact $\O(mdnt)$ for step 1, and bounded by $\O(dm^2+m^2k + mk^2)$ for step 2. As $n\geq m$ and $m\geq k$, the total complexity is bounded by $\O(mdnt + m^3)$.} $\O(mdnt+ m^3)$. For a summary of the analysis given in~\cite{yan_fast_2009}, let us consider the cluster memberships given by exact spectral clustering on $\sf{P}$ as well as the memberships given by exact spectral clustering on $\tilde{\sf{P}}=(p_1+\epsilon_1, \ldots, p_n+\epsilon_n)$ where the $\epsilon_i$ are any small perturbations on the initial points. Let us write $L$ (resp. $\tilde{L}$) the Laplacian matrix of the similarity graph on $\sf{P}$ (resp. $\tilde{\sf{P}}$). The analysis concentrates on the study of the misclustering rate $\rho$:
	\begin{align}
	\rho=\frac{\text{$\#$ of points with different memberships}}{n}.
	\end{align}
	The main result, building upon preliminary work in~\cite{NIPS2008_3480}, stems from a perturbation approach and reads:
	\begin{theorem}
		\label{thm:KASP}
		Under assumptions of Theorem 3 in~\cite{yan_fast_2009}, $\rho$ verifies: $\rho\leq \O\left(\frac{ k}{g_0^2} \|L-\tilde{L}\|_F\right)$ where $g_0$ is a value depending on the spectral gap. Also, under the assumptions of Theorem 6 in~\cite{yan_fast_2009}, one has, with high probability:
		\begin{align}
		\|L-\tilde{L}\|_F\leq \O\left(\sigma_{\epsilon}^{(2)} + \sigma_{\epsilon}^{(4)}\right)
		\end{align}
		with $\sigma_{\epsilon}^{(2)}$ and $\sigma_{\epsilon}^{(4)}$ the second and fourth moments of the perturbation's norms  $\|\epsilon_i\|$.
	\end{theorem}
	Combining both bounds, one obtains an upper bound on the misclustering rate that depends on the second and fourth moments of the perturbation's norms $\|\epsilon_i\|$. The ``collapse'' of points  onto the $m$ representative points, interpreted as a perturbation on the original points, should thus tend to minimize these two moments, leading the authors to propose \emph{distortion-minimizing} algorithms, such as KASP. A very similar algorithm, eSPEC, is described in~\cite{theeramunkong_approximate_2009}.

	\subsection{Other methods}
	\subsubsection{Approximate nearest neighbour search algorithms}
	The objective here is to approximate the nearest neighbour graph efficiently. Even though these methods are not necessarily based on sampling, we include them in the discussion as they are frequently used in practice.
	
	Given the feature vectors $p_i\in\mathbb{R}^d$ and a query point $\mathbf{q}\in\mathbb{R}^d$, the exact nearest neighbour search (exact NNS) associated to $\sf{P}$ and $\bf{q}$ is $p^*=\text{argmin}_{p\in\sf{P}}\;\text{dist}(\mathbf{q},p)$ where \emph{dist} stands for any distance. Different distances are possible depending on the choice of kernel $\kappa$. We will here consider the Euclidean norm as it enters the definition of the classical RBF kernel. Computing the exact NNS costs $\O(dn)$. The goal of the approximate NNS field of research is to provide faster algorithms that have the following control on the error. 
	\begin{definition}
		Point $p^*$ is an $\epsilon$-approximate nearest neighbor of query $\mathbf{q}\in\mathbb{R}^d$, if 
		$$\forall p\in\sf{P}\qquad\text{dist}(\mathbf{q}, p^{*})\leq(1 +\epsilon)\,\text{dist}(\mathbf{q}, p).$$ 
		For $\epsilon= 0$, this reduces to exact NNS. 
	\end{definition}
	Extensions of this objective to the $k$-nearest neighbour goal are considered in the NNS literature. A $k$-nearest neighbour graph can then be constructed simply by running an approximate $k$-NNS query for each object $p_i$. 
	Thus, approximate NSS algorithms are interesting candidates to approximate the adjacency matrix of the nearest-neighbour affinity graph, that we need in step 1 of spectral clustering. Many algorithms exist, their respective performances depending essentially on the dimension $d$ of the feature vectors. According to~\cite{avrithis_high-dimensional_2016}, randomized $k$-$d$ forests as implemented in the library FLANN~\cite{flann_pami_2014} are considered state-of-the-art for dimension of around $100$, whereas methods based on Balanced Box Decomposition (BBD)~\cite{arya_optimal_1998, anagnostopoulos_low-quality_2015} are known to perform well for $d$ roughly smaller than $100$. In high dimensions, to avoid the curse of dimensionality, successful approaches are for instance based on hashing methods (such as Locality Sensitive Hashing (LSH)~\cite{andoni_near-optimal_2006}, Product Quantization (PQ)~\cite{kalantidis_locally_2014}) or $k$-$d$ generalized random forests~\cite{avrithis_high-dimensional_2016}. Finally, proximity graph methods, that sequentially improve over a first coarse approximation of the $k$-NN graph (or other graph structures such as navigable graphs) have received a large attention recently and are becoming state-of-the-art in regimes where quality of approximation primes (see for instance~\cite{navigable, dong_efficient_2011, fu_fast_2017, aumuller2017ann}). 
	Such tools come with various levels of guarantees and computation costs, the details of which are not in the scope of this review. 
	
	Experimentally, to obtain an approximate $k$-NN graph with a typical recall rate\footnote{The recall rate for a node is the number of correctly identified $k$-NN divided by $k$. The recall rate for a $k$-NN graph is the average recall rate over all nodes.} of $0.9$, these algorithms are observed to achieve a complexity of $\O(dn^\alpha)$  with $\alpha$ close to $1$ ($\alpha\simeq1.1$ in~\cite{dong_efficient_2011} for instance). 
	
	\subsubsection{Feature selection and feature projection}
	\label{subsubsec:feature_selection}
	Some methods work on reducing the factor $d$ of the complexity $\O(dn^2)$ of the kernel computation via feature selection, 
	i.e., the sampling of features deemed more useful for the underlying clustering task, or feature projection, i.e., the projection on usually random subspaces of dimension $d'< d$. Feature selection methods are usually designed to \emph{improve} the classification by removing features that are too noisy or useless for the classification. We thus do not detail further these methods as they are not approximation algorithms \emph{per se}. The interested reader will find some entries in the literature via references~\cite{dash_feature_2009, he_laplacian_2006, zhao_spectral_2007, cai_unsupervised_2010}. Projection methods use random projections of the original points $\sf{P}$ on spaces of dimension $d'\sim\log{n}$ in order to take advantage of the Johnson-Lindenstrauss lemma of norm conservation: the kernel computed from the projected features in dimension $d'$ is thus an approximation of the true kernel with high probability. We refer to the works~\cite{perner_fast_2009, hunter_compressive_2010} for more details. 
	
	\section{Sampling given the similarity graph}
	\label{sec:spectral_embedding}

	We now suppose that the similarity graph is either given (e.g., in cases where the original data \emph{is} a graph) or has been well approximated (by approximate k-NN search for instance) and concentrate on sampling-based methods that aim to reduce the cost of computing the first $k$ eigenvectors of $L$. 

	These methods predominantly aim to approximate $L$ by a smaller matrix $\tilde{L}$ of size $m$. The eigen-decomposition is done in $\Rbb^m$ which can be significantly cheaper when $m \ll n$. In addition, each method comes with a fast way of lifting vectors from $\Rbb^m$ back to $\Rbb^n$ (this is usually a linear transformation). After lifting, the eigenvectors of $\tilde{L}$ are used as a proxy for those of $L$.
	
	 Unlike the previous section where a strong approximation guarantee of the exact embedding $U_k$ by an efficiently computed  $\tilde{U}_k$ was a distant and difficult goal to achieve in itself; we will see in this Section that the knowledge of the similarity graph not only enables to obtain such strong approximation guarantees, but also enables to control how the error on $U_k$ transfers as an error on the $k$-means cost. 
	
	 To be more precise, recall Eq.~\eqref{eq:def_k_means_cost} defining the $k$-means cost $f(\mathsf{C}; \mathsf{X})$ associated to the $n$ points $\sf{X}=(x_1, \ldots, x_n)$ and a centroid set $\sf{C}$.
	Now, suppose that we have identified a set of $n$ points $\tilde{\sf{X}}=(\tilde{x}_1|\ldots|\tilde{x}_n)$ that are meant to approximate the exact spectral embedding $\sf{X}$.  Moreover, let $\mathsf{C}^*$ (resp. $\tilde{\mathsf{C}}^*$) be the  optimal set of $k$ centroids minimizing  the $k$-means cost on ${\sf{X}}$ (resp. $\tilde{\sf{X}}$).  
	 We will see that several (not all) approximation methods of this Section achieve an end-to-end approximation guarantee of the form:
    $$
        \left| f(\mathsf{C^*}; \mathsf{X})^{\sfrac{1}{2}} - f(\mathsf{\tilde{C}^*}; \mathsf{X})^{\sfrac{1}{2}} \right| \leq \epsilon
    $$
    for some small $\epsilon$ with -at least- constant probability.  Such an end-to-end guarantee is indeed more desirable than a simple guarantee on the distance between $U_k$ and $\tilde{U}_k$: it informs us on the approximation quality of the attained clustering.


	\subsection{Nystr\"om-based methods}
	\label{subsec:Nystrom2}
	Once again, Nystr\"om-based methods are applicable. Let us concentrate on the choice $R=L_n$ to illustrate the main ideas. As explained in Section~\ref{subsec:Nystrom}, the  tailing $k$ eigenvectors  of $L_n$, $U_k$, can be interpreted as the top $k$ eigenvectors of the PSD  matrix $A=2I-L_n$. As such, the span of the top-$k$ eigenvectors of $\tilde{A}_k$, $\spanning{\tilde{U}_k}$,  obtained by running Algorithm 3 on $A$ should approximate the span of $U_k$. Now, how does ones goes from Nystr\"om theorems such as Theorem~\ref{thm:lev_score_Nystrom} to error bounds on the $k$-means cost function?\\
	
	
    \noindent The first step towards an end-to-end guarantee relies on the following result:
    
    \begin{lemma}[see the proof of Theorem 6 in~\cite{boutsidis2015spectral}]
    Denote by $\tilde{\mathsf{C}}^*$ the optimal centroid set obtained by solving $k$-means on the lines of $\tilde{U}_k$. It holds that
			\begin{align}
			\left| f(\mathsf{C^*}; \mathsf{X})^{\sfrac{1}{2}} - f(\mathsf{\tilde{C}^*}; \mathsf{X})^{\sfrac{1}{2}} \right| \leq 2 \|{\bf{E}}\|_F,
			\label{eq:kmeanscost1}
			\end{align}
			where 
			$
			{\bf{E}}=U_k U_k^\top - \tilde{U}_k\tilde{U}_k^\top.
			$
			\label{lemma:kmeans_end2end}
	\end{lemma}
    This means that the error made by considering the optimal $k$-means solution based on $\tilde{U}_k$ (instead of $U_k$) is controlled by the Frobenius norm of the projector difference $\mathbf{E}=U_k U_k^\top - \tilde{U}_k\tilde{U}_k^\top$. Furthermore, since\footnote{Based on three arguments: (i) for any two matrices $M_1$ and $M_2$ of rank $r_1$ and $r_2$ it holds that $\text{rank}(M_1+M_2)\leq r_1 + r_2$, (ii) for any matrix $M$ or rank $r$, $\|M\|_F\leq\sqrt{r}\|M\|_2$, and (iii) both $U_k$ and $\tilde{U}_k$ are of rank $k$.} $\|{\bf{E}}\|_F\leq\sqrt{2k}\|{\bf{E}}\|_2$ and $\|{\bf{E}}\|_2=\|\sin(\Theta(U_k, \tilde{U}_k))\|_2$, we can apply the Davis-Kahan $\sin\Theta$ perturbation theorem (see for instance Section VII of~\cite{bhatia_matrix_1997}) and, provided that $\sigma_k-\tilde{\sigma}_{k+1}>0$, obtain:
    $$\|{\bf{E}}\|_F\leq\sqrt{2k}\|{\bf{E}}\|_2\leq \sqrt{2k} \, \frac{\|A-\tilde{A}\|_2}{\sigma_k-\tilde{\sigma}_{k+1}},$$
    where $\{\sigma_i\}$ (resp. $\{\tilde{\sigma}_i\}$) are the singular values of $A$ (resp. $\tilde{A}$) ordered decreasingly\footnote{Note that, in our setting, $A=2I - L_n$ and $\sigma_k=2-\lambda_k$.}. 
    The final bound is obtained by combining the above with the leverage score sampling bound given by Theorem~\ref{thm:lev_score_Nystrom}: 
	%
	%
    \begin{theorem}
       Let $\tilde{U}_k$ be the eigenvectors obtained by running Alg.~3 on $A=2I-L_n$ (with the leverage score sampling scheme for the $m$ samples $S$ of step 1). Denote by $\tilde{\mathsf{C}}^*$ the optimal centroid set obtained by solving $k$-means on the lines of $\tilde{U}_k$. Then, for some constant $C>1$, we have
		\begin{align}
    		\left| f(\mathsf{C^*}; \mathsf{X})^{\sfrac{1}{2}} - f(\mathsf{\tilde{C}^*}; \mathsf{X})^{\sfrac{1}{2}} \right| \leq  2 \frac{\sqrt{2k}}{\sigma_k-\tilde{\sigma}_{k+1}}\left(\sigma_{k+1}(A) + \frac{Ck\log{(k/\delta)}}{m}\sum_{j=k+1}^n \sigma_j\right) \notag 
		\end{align}
		with probability at least $0.8 - 2\delta$.
		\label{theorem:nystrom_kmeans}
	\end{theorem}

    Examining the above bound one notices that $2 \sqrt{2k} \frac{\sigma_{k+1}(A)}{\sigma_k-\tilde{\sigma}_{k+1}}$ is independent of the number of samples. The incompressibility of this error term emanates from $A$ being (in general) different from its best low-rank approximation. On the other hand, all remaining error terms can be made independent of $k$ and $n$ by setting $$m=\O\left(k\sqrt{k}\log{k}\sum_{j=k+1}^n \frac{\sigma_j}{\sigma_k-\tilde{\sigma}_{k+1}}\right).$$
    This end-to-end guarantee is not satisfactory for several reasons. First of all, it relies on the assumption $\sigma_k>\tilde{\sigma}_{k+1}$, which is not necessarily true. Moreover, the Davis-Kahan theorem could in theory guarantee $\|\mathbf{E}\|_2\leq {\|A_k-\tilde{A}_k\|_2/\sigma_k}$ and  $\|\mathbf{E}\|_2\leq {\|A-\tilde{A}_k\|_2}/{\sigma_k}$, which are stronger than the bound depending on $\|A-\tilde{A}\|_2$ that we used. Unfortunately, Nyström approximation theorems do not give controls on 
    $\|A_k-\tilde{A}_k\|_2$ nor on $\|A-\tilde{A}_k\|_2$, impeding tighter end-to-end bounds. \\

	\subsection{Graph coarsening}
	
	Inspired by the algebraic multi-grid, researchers realized early on that a natural way to accelerate spectral clustering is by graph coarsening~\cite{hendrickson1995multi,karypis1998fast,dhillon2007weighted}. Here, instead of solving the clustering problem directly on $G$, one may first reduce it to a coarser graph $G_\c$ of dimension $m \ll n$ using a multi-level graph coarsening procedure. The expensive eigen-decomposition computation is done at a lower cost on the small dimension and the final spectral embedding is obtained by inexpensively lifting and refining the result. 
	
	In the notation of~\cite{loukas2018spectrally}, coarsening involves a sequence of $\c+1$ graphs  
	\begin{align}
	G = G_0 = (V_0, E_0, W_0) \quad G_1 = (V_1, E_1, W_1) \quad \cdots \quad G_{\c} = (V_{\c}, E_{\c}, W_{\c})
	\end{align}
	of decreasing size
	$n = n_0  >  n_{1} > \cdots > n_{\c} = m$, where each vertex of $G_{\ell}$ represents one of more vertices of $G_{\ell-1}$. 
	%
	%
	To express coarsening in algebraic form, we suppose that $L(G_{0})=L$ is the combinatorial Laplacian associated with $G$. We then obtain $L(G_\c)$ by applying the following repeatedly
	\begin{align}
	L(G_{\ell}) = P_\ell^{\mp} L(G_{\ell-1}) P_\ell^+,
	\end{align}
	where $P_\ell \in \Rbb^{n_\ell \times n_{\ell-1}}$ is a matrix with more columns than rows, $\ell = 1, 2, \ldots, \c$ is the level of the reduction and symbol ${\mp}$ denotes the transposed pseudoinverse. An eigenvector $\tilde{u} \in \Rbb^{m}$ of $L(G_\c)$ is lifted back to $\Rbb^n$ by backwards recursion $$\tilde{u}_{\ell-1} = P_{\ell} \tilde{u}_{\ell},$$ where $\tilde{u}_{\c} = \tilde{u}.$
	
	Matrices $P_1, P_2, \ldots, P_\c$ are determined by the transformation performed at each level.   
	Specifically, one should define for each level a surjective map $\varphi_{\ell}: V_{\ell-1} \rightarrow V_{\ell}$ between the original vertex set $V_{\ell-1}$ and the smaller vertex set $V_{\ell}$. We refer to the set of vertices $V_{\ell-1}^{(r)} \subseteq V_{\ell-1}$ mapped onto the same vertex $v_r'$ of $V_{\ell}$ as a \textit{contraction set}: $$V_{\ell-1}^{(r)} = \{v \in V_{\ell-1} : \varphi_{\ell}(v) = v'_r \}$$ 
	%
	It is easy to deduce from the above that {contraction sets} induce a partitioning of $V_{\ell-1}$ into $n_{\ell}$ subgraphs, each corresponding to a single vertex of $V_{\ell}$.
	%
	
	Then, for any $v'_r \in V_{\ell}$ and $v_i \in V_{\ell-1}$, matrices $P_{\ell} \in \Rbb^{n_{\ell} \times n_{\ell-1}}$  and $P^+_{\ell} \in \Rbb^{n_{\ell-1} \times n_{\ell}}$ are given by:
	\begin{align}
	P_{\ell}(r,i) = 
	\begin{cases} 
	\frac{1}{|V_{\ell-1}^{(r)}|} & \text{if } v_i \in V_{\ell-1}^{(r)} \\
	0       & \text{otherwise} 
	\end{cases}
	\quad \text{and} \quad
	P_{\ell}^+(i,r) = 
	\begin{cases} 
	1 \hspace{6mm} & \text{if } v_i \in V_{\ell-1}^{(r)} \\
	0       & \text{otherwise}. 
	\end{cases}
	\notag
	\end{align}
	The preceding construction is the only one that guarantees that every $L(G_{\ell})$ will be the combinatorial Laplacian associated with $G_\ell$~\cite{loukas2018graph}.  
	
	Note that from a computational perspective the reduction is very efficient and can be carried out in linear time: each coarsening level entails multiplication by a sparse matrix, meaning that $\O(e)$ and $\O(n)$ operations suffice, respectively, to coarsen $L$ and lift any vector (such as the eigenvectors of $L(G_\c)$) from $\Rbb^{m}$ back to $\Rbb^n$. 

	\subsubsection{Coarsening for spectral clustering}
	
	Using coarsening effectively boils down to determining for each $\ell$ how to partition $G_{\ell-1}$ into $n_\ell$ contraction sets $V_\ell^{(1)}, \ldots, V_\ell^{(n_{\ell})}$, such that, after lifting, the first $k$ eigenvectors $\tilde{U}_k$ of $L(G_{\c})$ approximate the spectral embedding $U_k$ derived from $L$. Alternatively, one may also solve the $k$-means problem in the small dimension and only lift the resulting cluster assignments~\cite{dhillon2007weighted}. This scheme is computationally superior but we will not discuss it here as it does not come with any guarantees.
	
	Perhaps the most \textit{simple} (and common) method of forming contraction sets is by the \emph{heavy edge matching heuristic}---originally developed in the multi-grid literature and first considered for graph partitioning in~\cite{karypis1998fast}. This method is derived based on the intuition that, the larger the weight of an edge, the less likely it will be that the vertices it connects will reside in different clusters. We should therefore aim to contract pairs of vertices connected by a heavy edge (i.e., of large weight) first.
	Let us consider this case further. By focusing on edges, we basically constrain ourselves by enforcing that every contraction set $V_\ell^{(r)}$ contains either two nodes connected by an edge, or a single node, signifying that said node is chosen to remain as is in the coarser graph. As such, we can reformulate the problem of selecting contraction sets at each level as that of selecting the largest number of edges (to attain the largest reduction), while also striving to make the cumulative sum of selected edge weights as large as possible (giving preference to heavy edges). This is exactly the \emph{maximum weight matching problem}, which can be approximated in linear time~\cite{duan2014linear}.
	
	A plethora of numerical evidence motivates the use of matching-based coarsening methods, such as the heavy-edge heuristic, for accelerating spectral clustering~\cite{karypis1998fast,dhillon2007weighted,safro2015advanced}. From a theoretical perspective, the approximation quality of matching-based methods was characterized in~\cite{loukas2018spectrally}. Therein, the matching was constructed in the following randomized manner: 
	\mybox{
		\textbf{Algorithm 5. Randomized edge contraction (one level)~\cite{loukas2018spectrally}}\\
		\textbf{Input:} A graph $G = (\sf{V}, \sf{E})$
		\begin{enumerate}[itemsep=-0.2ex]
			\item Associate with each $e_{ij} \in \sf{E}$ a probability $\textit{p}_{ij}>0$. 
			\item While $|E|>0$: 
			\begin{enumerate}[topsep=-1cm,itemsep=-0.2ex]
				\item Draw a sample $e_{ij}$ from $\sf{E}$ with probability $\propto \textit{p}_{ij}$.
				\item Remove from $\sf{E}$ both $e_{ij}$ as well as all edges sharing a common endpoint with it. 
				\item Construct contraction set $(v_i,v_j)$.
			\end{enumerate}   
		\end{enumerate}
		\textbf{Output:} Contraction sets
	}
	The following approximation result is known:
	%
	\begin{theorem}[Corollary 5.1 in~\cite{loukas2018spectrally}]
    Consider a graph with bounded degrees $d_i \ll n$ and $\lambda_k \leq \min_{e_{ij} \in E} \left\lbrace \frac{d_i + d_j}{2}\right\rbrace $. Suppose that the graph is coarsened by Algorithm~5, using a \emph{heavy-edge potential} such that $\textit{p}_{ij} \propto w_{ij}$. For sufficiently large $n$, a single level, and $\delta > 0$, 
    \begin{align}
     \left| f(\mathsf{\tilde{C}}^*; \mathsf{X})^{\sfrac{1}{2}} - f(\mathsf{C^*}; \mathsf{X})^{\sfrac{1}{2}} \right| 
        &= \O\left( \sqrt{\frac{1 - \frac{m}{n}}{\delta} \frac{\sum_{\ell=2}^k \lambda_\ell }{\lambda_{k+1} - \lambda_{k}} } \right) \notag
    \end{align}
    with probability at least $1- \delta$. Above, $\mathsf{\tilde{C}}^*$ is the optimal $k$-means solution when using the lifted eigenvectors of $L_\c$ as a spectral embedding. 
    \end{theorem}
    %
	%
	We deduce that coarsening works better when the spectral clustering problem is easy (as quantified by the weighted gap $ \sum_{\ell=2}^k \lambda_\ell / (\lambda_{k+1} - \lambda_{k}))$ and the achieved error is linear on the reduction ratio $1 - m/n$. 
	
	There also exist more {advanced} techniques for selecting contraction sets that come with stronger guarantees w.r.t. the attained reduction and quality of approximation, but feature running time that is not smaller than that of spectral clustering~\cite{loukas2018graph}. In particular, these work also with the normalized Laplacian and can be used to achieve multi-level reduction. Roughly, their strategy is to identify and contract sets  $\sf{S} \subset V$ for which $x(i) \approx x(j)$ for all vectors $x\in U_k$ and $v_i,v_j \in \sf{S}$. This strategy ensures that the best partitionings of $G$ are preserved by coarsening. We will not expand on these methods here as they do not aim to improve the running time of spectral clustering.   
	
	\subsection{Other approaches}
	
	In the following, we present two additional approaches for approximately computing spectral embeddings. The former can be interpreted as a sampling-based method (but in a different manner than the techniques discussed so far), whereas the latter is only vaguely linked to sampling. Nevertheless, we find that both techniques are very interesting and merit a brief discussion. 
	
	\subsubsection{Spectral sparsification} 
	%
	This approach is best suited for cases when the input of spectral clustering is directly a graph\footnote{When one starts from a set of points, it is preferable to sparsify the graph by retaining a constant number of nearest neighbors for each point. The resulting nearest neighbor graph has already $O(n)$ edges, which is the smallest possible.}.
	Differently from the methods discussed earlier, here the aim is to identify a matrix $\tilde{L}$ of the same size as $L$ but with fewer entries. Additionally, it should be ensured that 
	\begin{align}
	\frac{1}{1 + \epsilon} x^\top L x \leq x^\top \tilde{L} x \leq (1 + \epsilon) x^\top L x
	\end{align}
	for some small constant $\epsilon > 0$~\cite{spielman2011spectral}. 
	Most fast algorithms for spectral sparsification entail sampling $\O(n \log{n})$ edges from the total edges present in the graph. Different sampling schemes are possible~\cite{spielman2011graph,5671167}, but the most popular ones entail sampling edges with replacement based on their effective resistance. It should be noted that though computing all effective resistances exactly can be computationally prohibitive, the effective resistance of edges can be approximated in nearly linear-time on the number of edges based on a Johnson-Lindenstrauss argument~\cite{spielman2011graph}. 
	
	There are different ways to use sparsification in order to accelerate spectral clustering. The most direct one is to exploit the fact that the eigenvalues $\tilde{\lambda}_k$ and eigenvectors $\tilde{U}_k$ of $\tilde{L}$ approximate, respectively, the eigenvalues and eigenvectors of $L$ up to multiplicative error. This yields the same flavor of guarantees as in graph coarsening and ensures that the computational complexity of the partial eigen-decomposition will decrease when $e = \omega(n\log{n})$. A variation of this idea was considered in~\cite{wang2017towards}, though the latter did not provide a complete error and complexity analysis. 
	Alternative approaches are also possible. We refer the interested reader to~\cite{vishnoi2013lx} for a rigorous argument that invokes a Laplacian solver. 
	
	Despite these exciting developments, we should mention that the overwhelming majority of graph sparsification algorithms remain in the realm of theory. That is, we are currently not aware of any practical and competitive implementation and thus retain a measure of skepticism with regards to their utility in the setting of spectral clustering.    
	
	\subsubsection{Random eigenspace projection} 
	\label{subsec:projection-based}
	There also exists approaches that do not explicitly rely on sampling. The key starting point here is that, with regards to spectral clustering, one does not need the eigenvectors exactly---any rotation of $U_k$ suffices (indeed, $k$-means is an algorithm based on distances and rotations conserve distances).
	Even more generally, consider $\tilde{U}_k \in \Rbb^{n\times m}$ with $m \geq k$ and denote:
	$$ \epsilon = \min_{\mathbf{Q} \in \mathcal{Q}} \| U_k I_{k \times m} \mathbf{Q} - \tilde{U}_k \|_F, $$
	where $\mathcal{Q}$ is the space of $m \times m$ unitary matrices  and $I_{k \times m}$ consists of the first $k$ rows of an $m\times m$ identity  matrix. \\

	The following lemma (which is a generalization of Lemma~\ref{lemma:kmeans_end2end}) shows how $\epsilon$ can be used to provide control on the $k$-means error: 
	\begin{lemma}[Lemma 3.1 in~\cite{martin2018fast}]
		\label{lemma:kmeans_end2end2}
	Let $\tilde{\mathsf{C}}^*$ be the optimal solution of the $k$-means problem on $\tilde{U}_k$. It holds that\footnote{\textbf{A remark on notation}. Note that, here,  the lines $\tilde{X}$ of $\tilde{U}_k$ are points in dimension $m\geq k$, such that the optimal centroid set $\mathsf{\tilde{C}^*}$ minimizing the $k$-means cost on $\tilde{X}$ is a set of $k$ points in dimension $m\geq k$. In this context, the notation $f(\mathsf{\tilde{C}^*}; \mathsf{X})$ is ill-defined: it is a sum of distances between points that do not necessarily have the same dimension. We abuse notations and give the following meaning to $f(\mathsf{\tilde{C}}; \mathsf{X})$. 
	First, consider the matrix form of the $k$-means cost, as used in the proofs of Lemmas~\ref{lemma:kmeans_end2end} and~\ref{lemma:kmeans_end2end2}:
	$f(\mathsf{C}; \mathsf{X}) = \| X - \mathbf{C} \mathbf{C}^\top X \|^2_F$, 
	where $\mathbf{X}=(x_1|\ldots|x_n)^\top\in\mathbb{R}^{n\times k}$ and $\mathbf{C}\in \Rbb^{n\times k}$ is the (weighted) cluster indicator matrix associated to the Voronoi tesselation of $\sf{X}$ given $\sf{C}$: $
		\mathbf{C}_{i\ell} = 
		\frac1{\sqrt{s_\ell}}$ if data point $i$ belongs to cluster $\ell$, and 
		$0$ otherwise; 
	where $s_\ell$ is the size of cluster $\ell$. Now, let $\tilde{\mathbf{C}}\in \Rbb^{n\times k}$ be the cluster indicator matrix associated to the Voronoi tesselation of $\tilde{\sf{X}}$ given $\tilde{\sf{C}}$. One writes: 
	$f(\tilde{\mathsf{C}}; \mathsf{X})=\| X - \tilde{\mathbf{C}} \tilde{\mathbf{C}}^\top X \|^2_F.$
	}
			\begin{align}
		\left| f(\mathsf{C^*}; \mathsf{X})^{\sfrac{1}{2}} - f(\mathsf{\tilde{C}^*}; \mathsf{X})^{\sfrac{1}{2}} \right| \leq 2 \epsilon.
		\label{eq:kmeanscost2}
		\end{align}
			%
	\end{lemma}

	There exists (at least) two approaches to efficiently compute  $\tilde{U}_k$ while controlling $\epsilon$~\cite{boutsidis2015spectral,tremblay2016compressive} (see also related work in~\cite{halko_finding_2011}). We will consider here a simple variant of the one proposed in~\cite{tremblay2016compressive} and further analyzed in~\cite{martin2018fast}: Let $R \in \Rbb^{
		n\times m}$ be a random Gaussian matrix with centered i.i.d. entries, each having variance $\frac{1}{m}$. Further, suppose that we project $R$ onto $\spanning{U_k}$ by multiplying each one of its columns by an ideal projector $P_k$ defined as  
	\begin{equation}
	P_k = U \left(
	\begin{array}{cc}
	\mathbf{I}_k & 0\\
	0 & 0
	\end{array}
	\right) U^\top.
	\end{equation}
	%
	%
	\begin{theorem}[\cite{tremblay2016compressive, martin2018fast}]
		Let $\tilde{\mathsf{C}}^*$ be the optimal solution of the $k$-means problem on the lines of $\tilde{U}_k = P_k R$. For every $\delta \geq 0$, one has
		\begin{align} \displaystyle
		\left| f(\mathsf{C^*}; \mathsf{X})^{\sfrac{1}{2}} - f(\mathsf{\tilde{C}^*}; \mathsf{X})^{\sfrac{1}{2}} \right| \leq 2 \sqrt{\frac{k}{m}} \, (\sqrt{k}+\delta),
		\end{align}
		with probability at least $1-\exp(-\delta^2/2)$.
	\end{theorem} 
	This result means that for an ideal projector $P_k$, dimension $m = \O(k^2)$ suffices to guarantee good approximation (since the error becomes independent of $k$ and $n$)! A similar argument also holds when the entries of $R$, instead of being Gaussian, are selected i.i.d. from $\{-\sqrt{3},0,+\sqrt{3}\}$ with probabilities $\{1/6, 2/3, 1/6\}$, respectively~\cite{achlioptas2003database}. This construction has the benefit of being sparser and, moreover, is reminiscent of sampling. It should be noted that in~\cite{tremblay2016compressive}, $m=\O(\log{n})$ was deemed enough because one only wanted that the distance between two lines of $U_k$ was approximated by the distance between the same two lines of $\tilde{U}_k$. There was in fact no end-to-end control on the $k$-means error. 
	
	The discussion so far assumed that $P_k$ is an ideal projector onto $\spanning{U_k}$. However, in practice one does not have access to this projector as we are in fact in the process of \emph{computing} $U_k$. One may choose to approximate the action of $P_k$ by an application of a matrix function $h$ on $L$~\cite{tremblay_accelerated_2016, ramasamy_compressive_2015}. Assuming a point $\lambda_*$ in the interval $[\lambda_k, \lambda_{k+1})$ is known, one may select a polynomial~\cite{shuman2011chebyshev} or rational function~\cite{isufi2017autoregressive,7131465} that approximates the ideal low-pass response, i.e., $h(\lambda) = 1$ if $\lambda \leq \lambda_*$ and $h(\lambda) = 0$, otherwise. The approximated projector $\tilde{P}_k = h(L)$ can be designed to be very close to $P_k$. For instance, in the case of Chebychev polynomials of order $\c$ using the arguments of \cite[Lemma 1]{laurent2000adaptive} it is easy to prove that w.h.p. using $h(L)$ instead of $P_k$ does not add more than $\O(\c^{-\c}\sqrt{n})$ error in~\eqref{eq:kmeanscost2}. Furthermore, the operation $\tilde{P}_k R$ can conveniently be computed in $\O(m\c e)$ time via this polynomial approximation.
	
	The last ingredient needed for this approximation is $\lambda_*$, i.e., a point in the interval $[\lambda_k,\lambda_{k+1})$. Finding efficiently a valid $\lambda_*$ is difficult. An option is to rely on eigencount techniques~\cite{di2016efficient,paratte2016fast,puy_random_2016} to find one in%
	\footnote{\emph{Proof sketch}: Given $\lambda\in(0, \lambda_n]$, denote by $j$ the largest integer such that $\lambda_j\leq \lambda$ and by $P_j$ the orthogonal projector on $U_j$. Let $R \in \Rbb^{n\times m'}$ be a random Gaussian matrix with centered i.i.d. entries, each having variance $\frac{1}{m'}$ and denote by $\hat{j}=\|P_j R\|_F$. Relying on Theorem 4.1 (and the following discussion in Section 4.2) of~\cite{puy_random_2016} with $\sf{E}_\lambda=\mathbf{0}$, one has with prob. at least $1-\epsilon$ that $(1-\delta)j\leq\hat{j}\leq(1+\delta)j$ for all $j=1,\ldots, n$ provided $m'\geq \frac{1}{\delta^2}\log{\frac{n}{\epsilon}}$. Setting $\delta=1/(2k+3)$, gives w.h.p. that $\frac{2k+2}{2k+3}j\leq\hat{j}\leq\frac{2k+4}{2k+3}j$ for all $j=1, \ldots, n$ provided $m'\geq \O(k^2\log{n})$. This implies that w.h.p. for every $j\leq k+1$ it must be that $\text{round}(\hat{j})=j$, whereas when $j> k+1$ we have $\text{round}(\hat{j})>k+1$. Note that $\text{round}(\hat{j})$ is the closest integer to $\hat{j}$. By dichotomy on $\lambda\in(0, \lambda_n]$, one thus finds a $\lambda_*$ in time $\O(\c k^2(\log{n})(e + n \log(\lambda_n/(\lambda_{k+1} - \lambda_k))))$.} 
    $\O(\c k^2(\log{n})(e + n \log(\lambda_n/(\lambda_{k+1} - \lambda_k))))$ time, which features similar complexity as the Lanczos method (see discussion in Section~\ref{subsec:computational_complexity}). 
    Another option is to content oneself with values of $\lambda_*$ known only to be close to the interval $[\lambda_k,\lambda_{k+1})$, but thereby loosing the end-to-end guarantee~\cite{tremblay2016compressive}. 

	\section{Sampling in the spectral feature space}
	\label{sec:kmeans}
	Having computed (or approximated) the spectral embedding $\sf{X}=(x_1, x_2, \ldots, x_n)$, what remains is to solve the $k$-means problem on $\sf{X}$, in order to obtain $k$ centroids together with the associated $k$ classes obtained after Voronoi tessellation.
	
	The usual heuristic used to solve the $k$-means problem, namely the Lloyd-Max algorithm, is already very efficient as it runs in $\O(nk^2t)$ time as seen in Section~\ref{subsec:computational_complexity}. Nonetheless, 
	this section considers ways to accelerate $k$-means even further. In the following, we classify the relevant literature in five categories and point towards representative references for each case. In our effort to provide depth (as well as breadth) of presentation, the rest of the section details only methods that belong to the first and last categories.     
	
	\begin{itemize}
	    \item \textbf{Exact acceleration of Lloyd-Max.} There exists exact accelerated Lloyd-Max algorithms, some of them based on avoiding unnecessary distance calculations using the triangular inequality~\cite{hamerly_accelerating_2015, newling2016fast}, or on optimized data organization~\cite{kanungo_efficient_2002}, and others concentrating on clever initializations~\cite{arthur_k-means++:_2007, ostrovsky_effectiveness_2006}. The latter concern sampling and are discussed in Section~\ref{subsec:clever_initialization}.
	    
	    \item \textbf{Approximate acceleration of Lloyd-Max.} Approximately accelerating the Lloyd-Max algorithm has also received attention for instance via approximate nearest neighbour methods~\cite{philbin_object_2007}, via cluster closure~\cite{wang2015fast}, or via applying Lloyd-Max hierarchically (in the large $k$ context)~\cite{nister_scalable_2006}. An approach involving sampling is introduced in~\cite{sculley_web-scale_2010}: it is based on mini-batches sampled uniformly at random from $X$. We will not discuss further this method as it does not come with guarantees on the cost of the obtained solution.
	
    	\item \textbf{Methods involving sampling in Fourier.} There are a few sampling-based heuristics to solve the $k$-means problem, that are different from the Lloyd-Max algorithm. For instance, the work in~\cite{keriven_compressive_2017} proposes to sample in the frequency domain to obtain a sketch from which one may recover the centroids with an orthogonal matching pursuit algorithm specifically tailored to this kind of compressive learning task~\cite{gribonval_compressive_2017}. These methods are reminiscent of the RFF sketching approach introduced in Section~\ref{subsec:RFF}. We will not discuss them further. 
	
    	\item \textbf{Methods involving sampling features.} Similarly to ideas presented in Section~\ref{subsubsec:feature_selection} but here specific to the $k$-means setting, some works reduce the ambient dimension of the vectors, either by selecting a limited number of features~\cite{boutsidis_unsupervised_2009, altschuler_greedy_2016}, or by embedding all points in a lower dimension using random projections~\cite{boutsidis_randomized_2015, Cohen:2015, makarychev_performance_2018}. The tightest results to day are a $(1+\epsilon)$ multiplicative error on the $k$-means cost $f$ either by randomly selecting $\O(\epsilon^{-2}k\log{k})$ features or by projecting them on a random space of dimension $\O(\epsilon^{-2}\log{(k/\epsilon)})$ (sublinear in $k$!). The sampling result is useless in the spectral clustering setting as the ambient dimension of the spectral features is already $k$. The projection result could in principle be applied in our setting, to reduce the cost of the $k$-means step to $\O(tnk\log{k})$. We will nevertheless not discuss it further in this review.
	
    	\item \textbf{Methods involving sampling points.} Finally, the last group of existing methods are the ones that solve $k$-means on a subset $\sf{S}$ of $\sf{X}$, before lifting back the result on the whole dataset. We classify such methods in two categories. In Section~\ref{subsec:graph_agnostic}, we detail methods that are graph-agnostic, meaning that they apply to \textit{any} $k$-means problem; and in Section~\ref{subsec:graph_based} we discuss methods that explicitly rely on the fact that the features $x$ were in fact obtained from a known graph. We argue that the latter are better suited to the spectral clustering problem.
	\end{itemize}
	
	
	\subsection{Clever initialization of the Lloyd-Max algorithm}
	\label{subsec:clever_initialization}
	
	Recall that the $k$-means objective on $\sf{X}$ is to find the $k$ centroids $\sf{C}=(c_1, \ldots, c_k)$ that minimize the following cost function:
	\begin{align}
	f(\sf{C}; \sf{X}) = \sum_{x\in\sf{X}} ~\min_{c\in\sf{C}}\|x-c\|^2.
	\end{align}
	and that  $\sf{C}^*=\argmin_{\sf{C}} ~f(\sf{C}; \sf{X})$ is the optimal solution attaining cost $f^*=f(\sf{C}^*; \sf{X})$. Recall also that the Lloyd-Max algorithm (see Algorithm 2) converges to a local minimum of $f$, that we will denote by $\sf{C}_\text{lm}$, for which the cost function equals  $f_\text{lm}=f(\sf{C}_\text{lm}; \sf{X})$. It is crucial to note that the initialization of centroids $\sf{C}_\text{ini}$ in the first step of the Lloyd-Max algorithm, which usually is done by randomly selecting $k$ points in $\sf{X}$, is what determines the distance $|f^*-f_\text{lm}|$ to the optimal. As such, significant efforts have been devoted to smartly selecting $\sf{C}_\text{ini}$ by various sampling schemes.  
	
	As usual, we also face here the usual trade-off between sampling \textit{effectively} and \textit{efficiently}. The fastest sampling method is of course uniformly at random, but it does not come with any guarantee on the quality of the local minimum $\sf{C}_\text{lm}$ it leads to. An alternative sampling scheme, called $k$-means++ initialization, is based on the following more general \texttt{$D^2$-sampling} algorithm.
	\mybox{
		\textbf{Algorithm 6: \texttt{$D^2$-sampling}.} \\
		\textbf{Input}: $\sf{X}$, $m$ the number of required samples
		\begin{enumerate}[itemsep=-0.2ex]
			\item Initialize $\sf{B}$ with any $x$ chosen uniformly at random from $\sf{X}$.
			\item Iterate the following steps until $\sf{B}$ contains $m$ elements:
			\begin{enumerate}[topsep=-1cm,itemsep=-0.2ex]
				\item Compute $d_i=\min_{b\in\sf{B}} \|x_i-b\|^2$.
				\item Define the probability of sampling $x_i$ as $d_i/\sum_i d_i$.
				\item Sample $x_\text{new}$ from this probability distribution and add it to $\sf{B}$.
			\end{enumerate}
		\end{enumerate}
		\textbf{Output:} $\sf{B}$ a sample set of size $m$.
	}
	\noindent $k$-means++ initialization boils down to running Alg. 6 with $m=k$ to obtain a set of $k$ initial centroids. Importantly, when the Lloyd-Max heuristic is run with this initialization, the following guarantee holds:
	\begin{theorem}[\cite{arthur_k-means++:_2007}]
		For any set of data points, the cost $f_\text{lm}$ obtained after Lloyd-Max initialized with $k$-means++ is controlled in expectation: $\mathbb{E}(f_\text{lm}) \leq 8(\log k + 2)f^*$ .
	\end{theorem}
	In terms of computation cost, \texttt{$D^2$-sampling} with $m=k$ runs in $\O(nkd)$, that is, $\O(nk^2)$ in our setting of a spectral embedding $\sf{X}$ in dimension $k$. This work inspired other initialization techniques that come with similar guarantees and are in some cases faster~\cite{bahmani_scalable_2012, bachem_fast_2016}. The interested reader is referred to the review~\cite{celebi_comparative_2013} for further analyses on the initialization of $k$-means.

	\subsection{Graph agnostic sampling methods: coresets}
	\label{subsec:graph_agnostic} 
    The rest of Section~\ref{sec:kmeans}, considers sampling methods that fall in the following framework: (i)~sample a subset $\sf{S}$ of $\sf{X}$, (ii)~solve $k$-means on $\sf{S}$, (iii)~lift the result back on the whole dataset $\sf{X}$.  Section~\ref{subsec:graph_agnostic}, focuses on coresets: general sampling methods designed for any arbitrary $k$-means problem; whereas in Section~\ref{subsec:graph_based}, we will take into account the specific nature of the spectral features encountered in spectral clustering algorithms. 
	
	\subsubsection{Definition}
	Let $\sf{S}\subset\sf{X}$ be a subset of $\sf{X}$ of size $m$. To each element $s\in\sf{S}$ associate a weight $\omega(s)\in\mathbb{R}^+$. Define the estimated $k$-means cost associated to the weighted set $\sf{S}$ as:
	\begin{align}
	\label{eq:est_cost}
	\tilde{f}(\sf{C}; \sf{S}) = \sum_{s\in\sf{S}} \omega(s)\min_{c\in\sf{C}}\|s-c\|^2.
	\end{align}
	\begin{definition}[Coreset]
		Let $\epsilon\in (0,\frac{1}{2})$. The weighted subset $\sf{S}$ is a $\epsilon$-coreset for $f$ on $\sf{X}$ if, \emph{for every set $\sf{C}$}, the estimated cost is equal to the exact cost up to a relative error:
		\begin{align} \label{eq:coresets}
		\forall \sf{C}\qquad\left|\frac{\tilde{f}(\sf{C};\sf{S})}{f(\sf{C}; \sf{X})}-1\right|\leq\epsilon. 
		\end{align}
	\end{definition}
	This is the so-called ``strong'' coreset definition\footnote{A weaker version of this definition exists in the literature where the $\epsilon$-approximation is only required for $\sf{C}^*$.}, as the $\epsilon$-approximation is required for all $\sf{C}$. 
	The great interest of finding a coreset $\sf{S}$ comes from the following fact.
	Writing  $\tilde{\sf{C}}^*$ the set minimizing $\tilde{f}$, the following inequalities hold
	\begin{align*}
	(1-\epsilon) f(\sf{C}^*; \sf{X}) \leq  (1-\epsilon) f(\tilde{\sf{C}}^*; \sf{X})  \leq \tilde{f}(\tilde{\sf{C}}^*; \sf{S})\leq \tilde{f}(\sf{C}^*; \sf{S})\leq (1+\epsilon) f(\sf{C}^*; \sf{X}).
	\end{align*}
	The first inequality comes from the fact that $\sf{C}^*$ is optimal for $f$, the second and last inequality are justified by the coreset property of $\sf{S}$, and the third inequality comes from the optimality of $\tilde{\sf{C}}^*$ for $\tilde{f}$. This has two consequences:
	\begin{enumerate}
		\item First of all, since $\epsilon<\frac{1}{2}$:
		$$f(\sf{C}^*; \sf{X})\leq f(\tilde{\sf{C}}^*; \sf{X})\leq  (1+4\epsilon)f(\sf{C}^*; \sf{X}),$$
		meaning that $\tilde{\sf{C}}^*$ is a well controlled approximation of $\sf{C}^*$ with a multiplicative error on the cost.
		\item Estimating $\tilde{\sf{C}}^*$ can be done using the Lloyd-Max algorithm on the weighted subset\footnote{Generalizing Algorithm 2 to a weighted set is straightforward: in step 2b, instead of computing the center of each component, compute the weighted barycenter of each component.} $\sf{S}$, thus reducing the computation time from $\mathcal{O}(nk^2)$ to $\mathcal{O}(mk^2)$. 
	\end{enumerate}
	Coreset methods for $k$-means thus follow the general procedure:
	\mybox{
		\textbf{Algorithm 7. Coresets to avoid $k$-means on $\sf{X}$. }\\
		\noindent\textbf{Input:} $\sf{X}$, sampling set size $m$, and number of clusters $k \leq m $.
		\begin{enumerate}[itemsep=-0.2ex]
			\item Compute a weighted coreset $\sf{S}$ of size $m$ using a  \texttt{coreset-sampling} algorithm. 
			\item Run the Lloyd-Max algorithm on the weighted set $\sf{S}$ to obtain the set of $k$ centroids $\tilde{\sf{C}}$.
			\item ``Closest-centroid lifting": classify the whole dataset $\sf{X}$ based on the Voronoi cells of $\tilde{\sf{C}}$.
		\end{enumerate}
		\noindent\textbf{Output:} A set of $k$ centroids $\sf{C} = (c_1 , \ldots, c_k)$.
	}
	
	Coreset methods compete with one another on essentially two levels: the coreset size $m$ should be as small as possible in order to decrease the time of Lloyd-Max on $\sf{S}$, and the coreset itself should be sampled efficiently (at least faster than running $k$-means on the whole dataset!), which turns out in fact to be a strong requirement. The reader interested in an overview of coreset construction techniques is referred to the recent review~\cite{munteanu_coresets-methods_2017}. 
	
	\subsubsection{An instance of coreset-sampling algorithm}
	
	We focus on a particular coreset algorithm proposed in~\cite{bachem_practical_2017} that builds upon results developed in~\cite{langberg_universal_2010, feldman_unified_2011}: it is not state-of-the-art in terms of coreset size, but has the advantage of being easy to implement and fast enough to compute. It reads:
	\mybox{
		\textbf{Algorithm 8: a coreset sampling algorithm~\cite{bachem_practical_2017}}.\\
		\textbf{Input}: $\sf{X}$, $m$ the number of required samples, $t$ an iteration number
		\begin{enumerate}[itemsep=-0.5ex]
			\item Repeat $t$ times: draw a set of size $k$ using \texttt{$D^2$-sampling}. Out of the $t$ sets obtained, keep the set $\sf{B}$ that minimizes $f(\sf{B}; \sf{X})$. 
			\item $\alpha\leftarrow 16(\log{k}+2)$
			\item For each $b_\ell\in\sf{B}$, define $\sf{B}_\ell$ the set of points in $\sf{X}$ in the Voronoi cell of $b_\ell$
			\item Set $\phi = \frac{1}{n} f(\sf{B};\sf{X})$.
			\item For each $b_\ell\in\sf{B}$ and each $x\in\sf{B}_\ell$, define
			$$s(x) = \frac{\alpha}{\phi} \|x-b_\ell\|^2+\frac{2\alpha}{\phi|\sf{B}_\ell|}\sum_{x'\in\sf{B}_\ell} \|x'-b_\ell\|^2 + \frac{4n}{|\sf{B}_\ell|}$$
			\item Define the probability of sampling $x_i$ as $p_i=s(x_i)/\sum_x s(x)$
			\item $\sf{S}\leftarrow$ sample $m$ nodes i.i.d. with replacement from $p$ and associate to each sample $s$ the weight $\omega_s=\frac{1}{mp_s}$. 
		\end{enumerate}
		\textbf{Output:} A weighted set $\sf{S}$ of size $m$.
	}
	\noindent Theorem 2.5 of~\cite{bachem_practical_2017} states:
	\begin{theorem}
		Let $\epsilon\in(0, 1/4)$ and $\delta\in(0,1)$. Let $\sf{S}$ be the output of Alg.~8 with $t=\O(\log{1/\delta})$. Then, with probability at least $1-\delta$, $\sf{S}$ is a $\epsilon$-coreset provided that:
		\begin{align}
		\label{eq:min_m_coreset}
		m=\Omega\left(\frac{k^4\log{k}+k^2\log{1/\delta}}{\epsilon^2}\right).
		\end{align}
	\end{theorem}
	The computation cost of running this coreset sampling algorithm, running Lloyd-Max on the weighted coreset, and lifting the result back to $\sf{X}$ is dominated, when\footnote{To be precise, the statement holds if $n\geq\O\left(\frac{k^4}{\epsilon^2}\frac{\log{k}}{\log{1/\delta}}\right)$} $n\gg k$, by step 1 of Alg.~6 and thus sums up to $\O(n k^2\log{1/\delta})$. 
	
	\begin{remark}
		\emph{The coreset sampling strategy underlying this algorithm relies on the concept of sensitivity~\cite{langberg_universal_2010}. Many other constructions of coresets for $k$-means are possible~\cite{munteanu_coresets-methods_2017} with better theoretical bounds then Eq.~\eqref{eq:min_m_coreset}. Nevertheless, as the coreset line of research has been essentially theoretical, practical implementations of coreset-sampling algorithms are scarce. A notable exception is for instance the work in~\cite{frahling_fast_2008} that proposes a scalable hybrid coreset-inspired algorithm for $k$-means. Other exceptions are the sampling algorithms based on the farthest-first procedure, a variant of \texttt{$D^2$-sampling} that chooses each new sample to be $\argmax_i d_i$ instead of drawing it according to a probability proportional to $d_i$. Once $\sf{S}$ of size $m$ is drawn, then $\forall s\in\sf{S}$, each weight $\omega_s$ is set to be the cardinal of the Voronoi cell associated to $s$. Authors in~\cite{ros_protras:_2018} show that such weighted sets computed by different variants of the farthest-first algorithm are  $\epsilon$-coresets, but for values of $\epsilon$ that can be very large. For a fixed $\epsilon$, the number of samples necessary to have a $\epsilon$-coreset with this type of algorithm is unknown.}
	\end{remark}

	%
	

	\subsection{Graph-based sampling methods}
	\label{subsec:graph_based}
	
	The methods discussed so far in this section are graph agnostic both for the sampling procedure \emph{and} the lifting: 
	they do not take into account that, in spectral clustering, $\sf{X}$ are in fact spectral features of a known graph. 
	
	A recent line of work~\cite{tremblay2016compressive, martin2018fast, gadde_active_2014, gadde_active_2016-1} based on Graph Signal Processing (GSP)~\cite{shuman_emerging_2013-1,sandryhaila_big_2014} leverages this additional knowledge for accelerating both the sampling and the lifting steps. 
	For the purpose of the following discussion, define by $z_\ell\in\mathbb{R}^n$ the ground truth indicator vector of cluster $\ell$, i.e., $z_\ell(i)=1$ if node $i$ is in cluster $\ell$, and $0$ otherwise. The goal of spectral clustering is, of course, to recover $\{z_\ell\}_{\ell=1,\ldots,k}$. 
	
	Broadly, GSP-based methods can be summarized in the following general methodology~\cite{tremblay2016compressive}:
	
	
	\mybox{
		\textbf{Algorithm 9. Graph-based sampling strategies to avoid $k$-means on $\sf{X}$. }\\
		\textbf{Input}: $\sf{X}$, $m$ the number of required samples, $k$ the number of desired clusters
		\begin{enumerate}[itemsep=-0.2ex]
			\item Choose the random sampling strategy. Either:
			\begin{enumerate}[topsep=-1cm,itemsep=-0.2ex]
				\item \textbf{uniform (i.i.d.)} Draw $m$ i.i.d. samples uniformly. 
				\item \textbf{leverage score (i.i.d.)} Compute $\forall x_i, \pp_i^*=\|U_k^\top\delta_i\|^2 / k$. Draw $m$ i.i.d. samples from $p^*$. (optional:) set the weight of each sample $s$ to $1/\pp_s^*$.  
				\item \textbf{DPP} Sample a few times independently from a DPP with kernel $K_k=U_k U_k^\top$. (optional:) set the weight of each sample $s$ to $1/\pi_s$.  
			\end{enumerate}
			\item Run the Lloyd-Max algorithm on the (possibly weighted) set $\sf{S}$ to obtain  the $k$ reduced cluster indicator vectors $z^r_\ell\in\mathbb{R}^m$.
			\item Lift each reduced indicator vector $\{z^r_\ell\}_{\ell=1,\ldots, k}$ to the full graph either with
			\begin{enumerate}[topsep=0pt,itemsep=-0.2ex]
				\item \textbf{Least-square} Solve Eq.~\eqref{eq:unbiased_solution} with $y\leftarrow z^r_\ell$.
				\item \textbf{Tikhonov} Solve Eq.~\eqref{eq:SSL_solution} with $y\leftarrow z^r_\ell$.
			\end{enumerate}
			In both cases, $P_\sf{S}$ should be set to $\frac{1}{N}I_m$ if uniform sampling was chosen, to $\text{diag}(\pp^*_{s_1}, \ldots \pp^*_{s_m})$ if leverage score sampling was chosen, and to $\text{diag}(\pi_{s_1}, \ldots \pi_{s_m})$ if DPP sampling was chosen. 
			\item Assign each node $j$ to the cluster $\ell$ for which $\hat{z}_\ell(j)/\|\hat{z}_\ell\|_2$ is maximal.
		\end{enumerate}
		\textbf{Output:} A partition of $\sf{X}$ in $k$ clusters 
	}
	
	To aid understanding, let us start by a high-level description of Algorithm 9. 
	The indicator vectors $z_\ell$ are interpreted as graph signals that are (approximately) bandlimited on the similarity graph $G$ (see Section~\ref{subsec:GSP} for a precise definition). As such, there is no need to measure these indicator vectors everywhere: one can take advantage of generalized Shannon-type sampling theorems to select the set $\sf{S}$ of $m$ nodes to measure (step 1). Then $k$-means is performed on $\sf{S}$ to obtain the indicator vectors $z_\ell^r\in\mathbb{R}^m$ on the sample set $\sf{S}$ (step 2). These reduced indicator vectors are interpreted  as noisy measurements of the global cluster indicator vectors $z_\ell$ on $\sf{S}$. The solutions $z_\ell^r$ are lifted back to $\sf{X}$ as $\hat{z}_\ell$ via solving an inverse problem taking into account the bandlimitedness assumption or via label-propagation on the graph structure reminiscent of semi-supervised learning techniques (step 3). As the lifted solutions $\hat{z}_\ell$ do not have a binary structure as true indicator vectors should have, an additional assignment step is necessary: assign each node $j$ to the class $\ell$ for which $\frac{\hat{z}_\ell(j)}{\|\hat{z}_\ell\|_2}$ is maximal (step 4). 
	
	The rest of this section is devoted to the discussion of the three sampling schemes as well as the two lifting procedures considered in this framework. To this end, we will first introduce a few graph signal processing (GSP) concepts in Section~\ref{subsec:GSP} before discussing in Section~\ref{subsec:graph_sampling} several examples of graph sampling theorems appropriate to the spectral clustering context.
	
	\subsubsection{A brief introduction to graph signal processing (GSP)}
	\label{subsec:GSP}
	Denote by $U=(u_1|\ldots|u_n)\in\mathbb{R}^{n\times n}$ the matrix of orthonormal eigenvectors of the Laplacian matrix $L$, with the columns ordered according to their associated sorted eigenvalues: $0=\lambda_1\leq\lambda_2\leq\ldots\leq\lambda_n$.  
	In the GSP literature~\cite{shuman_emerging_2013-1,sandryhaila_big_2014}, these eigenvectors are interpreted as graph Fourier modes for two main reasons:
	\begin{itemize}
		\item By analogy to the ring graph, whose Laplacian matrix is exactly the (symmetric) double derivative discrete operator, and is thus diagonal in the basis formed by the classical 1D discrete Fourier modes. 
		\item A variational argument stemming from the Dirichlet form can be exploited to express eigenvectors $u_i$ of $L$ as the basis of minimal variation $x^\top L x = \frac{1}{2}\sum_{ij} W_{ij}\left[x(i)-x(j)\right]^2$ on $G$ and eigenvalues $\lambda_i$ as a sum of local variations of $u_i$, i.e., a generalized graph frequency.
	\end{itemize}
	A \emph{graph signal} $z\in\mathbb{R}^n$ is a signal that is defined on the nodes of a graph: its $i$-th component is associated to node $i$. Given the previous discussion, the graph Fourier transform of $z$, denoted by $\tilde{z}$, is its projection on the graph Fourier modes: $\tilde{z}=U^\top z\in\mathbb{R}^n$. The notion of graph filtering naturally follows as a multiplication in the Fourier space. More precisely, define a real-valued filter function $h(\lambda)$ defined on $[0, \lambda_n]$. The signal $x$ filtered by $h$ reads $Uh(\Lambda)U^\top x$, where we use the convention $h(\Lambda)=\text{diag}(h(\lambda_1), h(\lambda_2), \ldots, h(\lambda_n))$. In the following, we will use the following notation for graph filter operators:
	\begin{align}
	\label{eq:filtering_shorthand}
	h(L) = Uh(\Lambda)U^\top.
	\end{align}
	For more details on the graph Fourier transform and filtering, their various definitions and interpretations, we refer the reader to~\cite{TREMBLAY2018299}. 
	
	Of interest for the discussion in this paper, one may define bandlimited graph signals as linear combinations of the first few low-frequency Fourier modes. Writing  $U_k=(u_1|\ldots|u_k)\in\mathbb{R}^{n\times k}$, we have the formal definition:
	\begin{definition}[$k$-bandlimited graph signal] A graph signal $z \in \mathbb{R}^{n}$ is $k$-bandlimited if $z \in \spanning{U_k}$, i.e.,  $\exists~\mathbf{\alpha}\in\mathbb{R}^k$ such that $z = U_k\mathbf{\alpha}$.
	\end{definition}
	To grasp why the notion of $k$-bandlimitedness lends itself natually to the approximation of spectral clustering, consider momentarily a graph with $k$ disconnected components and $z_\ell\in\mathbb{R}^n$ the indicator vector of component $\ell$. It is a well known property of the (combinatorial) Laplacian that $\{z_\ell\}_{\ell=1,\ldots,k}$ form a set of orthogonal eigenvectors of $L$ associated to eigenvalue $0$: that is, the set of indicator vectors $\{z_\ell\}_{\ell=1,\ldots,k}$ form a basis of $\spanning{U_k}$. Understanding arbitrary graphs with block structure as a perturbation of the ideal disconnected component case, the indicator vectors $\{z_\ell\}_{\ell=1,\ldots,k}$ of the blocks should live close to $\spanning{U_k}$ (in the sense that the difference between any $z_\ell$ and its orthogonal projection onto $\spanning{U_k}$ is small). This in turn implies that every $z_\ell$ should be approximately $k$-bandlimited. 
	
	As we will see next, the bandlimitedness assumption is very useful because it enables us to make use of generalized versions of Nyquist-Shannon sampling theorems, taking into account the graph. 
	
	\subsubsection{Graph sampling theorems}
	\label{subsec:graph_sampling}

	The periodic sampling paradigm of the Shannon theorem for classical bandlimited signals does not apply to graphs without specific regular structure. In fact, a number of sampling schemes have been recently developed with the purpose of generalizing sampling theorems to graph signals~\cite{sakiyama_eigendecomposition-free_2018, chen_discrete_2015, tsitsvero_signals_2016, puy_random_2016} (see~\cite{LORENZO2018261} for a review of existing schemes). 
	
	Let us introduce some notations. 
	Sampling entails selecting a set $\sf{S}=(s_1,\ldots,s_m)$ of $m$ nodes of the graph. To each possible sampling set, we associate a measurement matrix $M=(\mathbf{\delta}_{s_1}|\mathbf{\delta}_{s_2}|\ldots|\mathbf{\delta}_{s_m})^{\top} \in\mathbb{R}^{m\times n}$ where $\mathbf{\delta}_{s_i}(j)=1$ if $j=s_i$, and 0 otherwise. Now, consider a 
	$k$-bandlimited signal $z\in\spanning{U_k}$. The measurement of $z$ on $\sf{S}$ reads:
	\begin{align}
	\label{eq:system_rec}
	    y = M z + \mathbf{n} \in\mathbb{R}^m,
	\end{align}
	where $\mathbf{n}$ models measurement noise. 
	The sampling question boils down to: how should we sample $\sf{S}$ such that one can recover any bandlimited $z$ given its  measurement $y$? There are three important components to this question: (i)~how many samples $m$ do we allow ourselves ($m=k$ being the strict theoretical minimum)? (ii)~how much does it cost to sample? (iii)~how do we in practice recover $z$ from $y$ and how much does that inversion cost?
	
	There are a series of works that propose greedy algorithms to find the ``best" set $\sf{S}$ of minimal size $m=k$ that embed all $k$-bandlimited signals (see for instance~\cite{tremblay_graph_2017} and references therein). These algorithms cost $\mathcal{O}(nk^4)$ and are thus not competitive in our setting\footnote{It takes longer to find a good sample than to run $k$-means on the whole dataset!}. Moreover, in our case, we don't really need to be that strict on the number of samples and can allow more than $k$ samples. A better choice is to use random graph sampling techniques. In the following we consider two types of independent sampling (uniform and leverage-score sampling) as well as a more involved method based on determinantal point processes. \\

	\noindent\textbf{Independent sampling.} In the i.i.d. setting, one defines a discrete probability distribution $p\in\mathbb{R}^n$ over the node set $\sf{V}$. The sampling set $\sf{S}$ is then generated by drawing $m$ nodes independently with replacement from $p$. At each draw, the probability to sample node $i$ is denoted by $\pp_i$. We have $\sum_i \pp_i=1$ and write $P=\text{diag}(p)$. Under this sampling scheme, the following Restricted Isometry Property holds for the associated measurement matrix $M$~\cite{puy_random_2016}. 
	
	\begin{theorem}
		For any $\delta, \epsilon \in (0,1)$, with probability at least $1-\delta$:
		\begin{equation}
		(1-\epsilon)\|z\|^2_2 \leq \frac{1}{m}\|MP^{-1/2}z\|^2_2\leq(1 +\epsilon)\|z\|^2_2
		\end{equation}
		for all $z\in\spanning{U_k}$ provided that 
		\begin{equation}
		\label{eq:m_min}
		m\geq \frac{3}{\epsilon^2}(\nu^k_p)^2\log{\frac{2k}{\delta}}
		\end{equation}
		where $\nu^k_p$ is the so-called graph weighted coherence:
		\begin{equation}
		\nu^k_p =  \max_i \left\{\pp_i^{-1/2}\|U_k^\top\delta_i\|_2\right\}.
		\end{equation}
	\end{theorem}
	This property is important as it says, in a nutshell, that any two different bandlimited signals will be identifiable post-sampling provided the number of samples is large enough. The concept of large enough depends on $(\nu^k_p)^2$: a measure of the interplay between the probability distribution and the norms of the lines of $U_k$. In the uniform i.i.d. case since $\pp_i=1/n$, one has $(\nu^k_p)^2=n\max_i \|U_k^\top\delta_i\|^2_2$, which stays under control only for very regular graphs, but can be close to $n$ in irregular graphs such as the star graph. The good news is that there exists an optimal sampling distribution (in the sense that it minimizes the right-hand side of inequality~\eqref{eq:m_min}) that adapts to the graph at hand:
	\begin{align}
	\label{Eq:lev_score}
	\pp_i^* = \frac{\|U_k^\top\delta_i\|_2^2}{k}
	\end{align}
	In fact, in this case, $(\nu^k_{p^*})^2$ matches its lower bound $k$ and the necessary number of samples $m$ to embed all bandlimited signals drops to $\mathcal{O}(k\log{k})$. The distribution $p^*$ is also referred to by the name ``leverage scores'' in parts of the literature (see discussion in Section~\ref{subsec:Nystrom})~\cite{drineas2018lectures}. As such, i.i.d. sampling under $p^*$ will be referred to as leverage score sampling.\\
	
	\noindent Now, for lifting, there are several options.
	\begin{itemize}
		\item If one uses the unbiased decoder
		\begin{align}
		\label{eq:unbiased_solution}
		\hat{z} &= \argmin_{\mathbf{w}\in\spanning{U_k}} \|P^{-1/2 }_{\sf{S}}(M\mathbf{w}-y)\|^2
		\end{align}
		where $P^{-1/2 }_{\sf{S}}=MP^{-1/2 }M^\top$, 
		then the following reconstruction result holds~\cite{puy_random_2016}:
		\begin{theorem}
			Let $\sf{S}$ be the i.i.d. nodes sampled with distribution $p$ and $M$ be the associated sampling matrix. Let  $\epsilon,\delta\in(0,1)$ and suppose that $m$ satisfies Eq.~\eqref{eq:m_min}. With probability at least $1-\delta$, for all $z\in\spanning{U_k}$ and ${\bf{n}}\in\mathbb{R}^m$, the solution $\hat{z}$ of Eq.~\eqref{eq:unbiased_solution} verifies:
			$$\|\hat{z}-z\|_2\leq\frac{2}{\sqrt{m(1-\epsilon)}}
			\|P^{-1/2}_{\sf{S}}{\bf{n}}\|_2.$$
		\end{theorem}
	This means that a noiseless measurement of a $k$-bandlimited signal yields a perfect reconstruction. Also, this quantifies how increasing $m$ reduces the error of reconstruction due to a noisy measurement. Note that this error may be large if there is a significant measurement noise on a node that has a low probability of being sampled. However, by definition, this is not likely to happen.
		
		\item One can also use a label-propagation decoder reminiscent to semi-supervised learning techniques~\cite{bengio_label_2006, chapelle_semi-supervised_2010}:
		\begin{align}
		\label{eq:SSL_solution}
		\hat{z} &= \argmin_{\mathbf{w}\in\mathbb{R}^n} \|P^{-1/2 }_{\sf{S}}(M\mathbf{w}-y)\|^2 + \gamma\; \mathbf{w}^\top g(L)\mathbf{w},
		\end{align}
		where $\gamma$ is a regularization parameter, $g(L)$ a graph filter operator as in Eq.~\eqref{eq:filtering_shorthand} with  $g(\lambda)$ a non-decreasing function.  As $g$ is non-decreasing, the regularization term of Eq.~\eqref{eq:SSL_solution} penalizes high frequency solutions, that is, solutions that are not smooth along paths of the graph. 
		Theorems controlling the error of reconstruction are more involved and we refer the reader to Section 3.3 of~\cite{puy_random_2016} for details.
		\item Other decoders~\cite{bellec2017, pena} are in principle possible, replacing for instance the $\ell_2$ Laplacian-based regularization $\mathbf{w}^\top g(L)\mathbf{w}$ by $\ell_1$-regularizers $\|\nabla \mathbf{w}\|_1$, but they come with an increased computation cost, lesser guarantees, and have not been used for spectral clustering: we will thus not detail them further. 
	\end{itemize}

	Let us discuss the computation costs of the previous sampling and lifting techniques. In terms of sampling time, uniform sampling is obviously the most efficient and runs in $\O(k)$. Leverage score sampling is dominated by the computation of the optimal sampling distribution $p^*$ of Eq.~\eqref{Eq:lev_score}, which takes $\O(nk)$ time\footnote{Note that the complexity is different from the leverage score computation of the Nyström techniques of Sections~\ref{subsec:Nystrom} and~\ref{subsec:Nystrom2} because, here, we suppose $U_k$ known whereas $U_k$ was not known in the previous sections. With $U_k$ known, computing the leverage scores only entails computing the normalized energy of each line of $U_k$.}. In terms of lifting time, solving the decoder of Eq.~\eqref{eq:unbiased_solution} costs $\O(nk + mk^2)$. Solving the decoder of Eq.~\eqref{eq:SSL_solution} costs $\O(et)$ via the conjugate gradient method, where $t$ is the iteration number of the gradient solver (usually around 10 or 20 iterations suffice to obtain good accuracy when $g(L) = L$). 
	
	This discussion calls for a few remarks. First of all, these theorems are valid if we suppose that $z$ is exactly $k$-bandlimited, which is in fact only an approximation if we consider $z$ to be the ground truth indicator vectors of the $k$ clusters to detect in the spectral clustering context. In this case, we can always decompose $z$ as the sum of its orthogonal projection onto $\spanning{U_k}$  and its complement $\mathbf{\beta}$: $z=U_k U_k^\top z + \mathbf{\beta}$. Eq.~\eqref{eq:system_rec} becomes $y=M U_k U_k^\top z + \mathbf{n}$ where $\mathbf{n}$ now represents the sum of a measurement noise and the distance-to-model term $M\mathbf{\beta}$. The aforementioned theorems can then  be applied to $U_k U_k^\top z$. Moreover, note that the decoder of Eq.~\eqref{eq:SSL_solution} is not only faster than the other ones in general, it also does not constrain the solution $\hat{z}$ to be exactly in $\spanning{U_k}$, which is in fact desirable in the spectral clustering context: we thus advocate for the decoder of Eq.~\eqref{eq:SSL_solution}.\\
	
	\noindent\textbf{DPP sampling.} Determinantal Point Processes are a class of correlated random sampling strategies that strive to increase ``diversity" in the samples, based on a kernel $K$ expliciting the similarity between variables. DPP sampling has been used successfully in a number of applications in machine learning (see for instance~\cite{kulesza_determinantal_2012}).
	
	Denote by $[n]$ the set of all subsets of $\{1,2,\ldots,n\}$. An element of $[n]$ could be the empty set, all elements of $\{1,2,\ldots,n\}$ or anything in between. DPPs are defined as follows:
	
	\begin{definition}[Determinantal Point Process~\cite{kulesza_determinantal_2012}] 
		\label{def:DPP} Consider a point process, i.e., a process that randomly draws an element $\sf{S}\in[n]$. It is determinantal if, $\forall\;\sf{A}\subseteq\sf{S}$, 
		$$\mathbb{P}(\sf{A}\subseteq\sf{S}) = \text{det}(K_{\sf{A}}),$$
		where $K\in\mathbb{R}^{n\times n}$, a semi-definite positive matrix $0\preceq K\preceq 1$, is called the marginal kernel; and $K_\sf{A}$ is the restriction of $K$ to the rows and columns indexed by the elements of $\sf{A}$. 
	\end{definition}
	The marginal probability  $\pi_i$ of sampling an element $i$ is thus $K_{ii}$.
	Consider the following projective kernel: 
	\begin{align}
	\label{def:Kk}
	K_k = U_k U_k^\top. 
	\end{align}
	One can show that DPP samples from such projective kernels are necessarily of size $k$. After measuring the $k$-bandlimited signal $z$ on a DPP sample $\sf{S}$, one has the choice between the same decoders as before (see  Eqs.~\eqref{eq:unbiased_solution} and~\eqref{eq:SSL_solution}). For instance:
	\begin{theorem}
		\label{thm:Kk} For all $z\in\spanning{U_k}$, let $y=M z + \mathbf{n} \in\mathbb{R}^k$ be a noisy measurement of $z$ on a DPP sample obtained from kernel $K_k$. The decoder of Eq.~\eqref{eq:unbiased_solution} with $P=\text{diag}(\pi_1, \ldots, \pi_n)$ necessarily enables perfect reconstruction up to the noise level. Indeed, one obtains:
		\begin{align}
		\label{eq:DPP_rec_guarantee}
		\|\hat{z}-z\|_2\leq\frac{1}{\sqrt{\lambda_{\text{min}}\left(U_k^\top M^\top P^{-1}_\sf{S}MU_k\right)}}
		\|P^{-1/2}_{\sf{S}} \mathbf{n}\|_2.
		\end{align}
	\end{theorem}
	\begin{proof}
		The proof is only  partly in~\cite{tremblay_graph_2017} and we complete it here. Let us write $z=U_k \alpha$. Solving Eq.~\eqref{eq:unbiased_solution} entails computing $\hat{\alpha}\in\mathbb{R}^k$ s.t. $\|P^{-1/2 }_{\sf{S}}(MU_k\hat{\alpha}-y)\|^2$ is minimal. Setting the derivative w.r.t. $\hat{\alpha}$ to $0$, and replacing $y$ by $MU_k\alpha+\mathbf{n}$, yields:
		$$U_k^\top M^\top P^{-1}_\sf{S} MU_k\hat{\alpha} = U_k^\top M^\top P^{-1}_\sf{S} MU_k\alpha + U_k^\top M^\top P^{-1}_\sf{S} \mathbf{n}.$$
		Recall that $\sf{S}$ is a sample from a DPP with kernel $K_k$: $\text{det}(MU_k U_k^\top M^\top)$ is thus strictly superior to $0$, which implies that $MU_k$ is invertible, which in turn implies that $\hat{\alpha} = \alpha + (MU_k)^{-1} \mathbf{n}$. One thus has $\|\hat{z} - z\|_2 = \|\hat{\alpha} - \alpha\|_2=\left\|(MU_k)^{-1} \mathbf{n}\right\|_2=\left\|(P^{-1/2}_\sf{S}MU_k)^{-1} P^{-1/2}_\sf{S}\mathbf{n}\right\|_2$. Using the matrix $2$-norm to bound this error yields 
		$$\|\hat{z} - z\|_2\leq \sqrt{\lambda_{\text{max}}\left[\left(U_k^\top M^\top P^{-1}_\sf{S} MU_k\right)^{-1}\right]} \|P^{-1/2}_\sf{S} \mathbf{n}\|_2,$$
		as claimed.
	\end{proof}
	Several comments are in order:
	\begin{itemize}
		\item The particular choice of kernel $K_k = U_k U_k^\top$ implies that the marginal probability of sampling node $i$,  $\pi_i=\|U_k^\top\delta_i\|^2$, is proportional to the leverage scores $\pp_i^*$. The major difference between the i.i.d. leverage score approach and the DPP approach comes from the negative correlations induced by the DPP. In fact, the probability of jointly sampling nodes $i$ and $j$ in the DPP case is $\pi_i \pi_j - K_{ij}^2=\pi_i \pi_j - (\delta_i^\top U_k U_k^\top \delta_j)^2$. The interaction term $(\delta_i^\top U_k U_k^\top \delta_j)^2$ will be typically large if $i$ and $j$ are in the same cluster, and small if not. In other words, different from the i.i.d. leverage score case where each new sample is drawn regardless of the past, the DPP procedure avoids to sample nodes containing redundant information. 
		\item Whereas the leverage score approach only guarantees a RIP with high probability after $\O(k\log{k})$ samples, the DPP approach has a stronger deterministic guarantee: it enables perfect invertibility (up to the noise level) after precisely $m=k$ samples. The reconstruction guarantee of Eq.~\eqref{eq:DPP_rec_guarantee} is nevertheless not satisfactory: even corrected by the marginal probabilities $P_\sf{S}$, the matrix $U_k^\top M^\top P^{-1}_\sf{S}MU_k$ can still have a very small $\lambda_\text{min}$, such that reconstruction may be quite sensitive to noise. Improving this control is still an open problem. In practice, sampling independently $2$ or $3$ times from a DPP with kernel $K_k$, creates a set $\sf{S}$ of size $2k$ or $3k$ that is naturally more robust to noise.
		\item Whereas independent sampling is straightforward, sampling from a DPP with arbitrary kernel costs in general $\O(n^3)$ (see Alg. 1 of~\cite{kulesza_determinantal_2012} due to~\cite{hough_determinantal_2006}). Thankfully, in the case of a projective kernel such as $K_k$, one can sample a set in $\O(nk^2)$ based on Alg. 3 of~\cite{tremblay_optimized_2018}.
	\end{itemize}

	\section{Perspectives}
	\label{sec::perspectives}
	
	Almost two decades have passed since spectral clustering was first introduced. Since then, a large body of work has attempted to accelerate its computation. \textit{So, has the problem been satisfactorily addressed?} -- or, despite all these works,  is there still room for improvement and further research?  
	
	To answer, we must first define what ``satisfactorily addressed'' would entail.  As we have seen, the prototypical spectral clustering algorithm can be divided in three sub-problems: the similarity graph computation running in $\O(dn^2)$, the spectral embedding computation running in $\O(t(ek + nk^2))$ with an Arnoldi algorithm with implicit restart, and the $k$-means step running in $\O(tnk^2)$. Our criteria for evaluating an approximation algorithm aiming to accelerate one (or more) of these sub-problems are two-fold:
	
	\begin{itemize}
	    \item We ask that the approximation algorithm's computation cost is effectively lighter than the cost of the sub-problem(s) it is supposed to accelerate! The ultimate achievement is an order-of-magnitude improvement w.r.t. $n$ (or $e$), $d$ and/or $k$, especially when the complexity has no hidden constants (i.e., the algorithm is practically implementable). When such a gain is not possible, a gain on the constants of the theoretical cost is also considered worthwhile.
	
	
	\item The algorithm should come with convincing guarantees in terms of the quality of the found solution. Heuristics or partially motivated methods do not cut it. We require that, under \textit{mild assumptions}, the proposed solution is \emph{provably close to} the exact solution. Let us clarify two aspects of this statement further:
	\begin{itemize}
		\item It is difficult to concretely classify assumptions as mild, but a useful rule of thump is checking whether the theoretical results are meaningful for the significant majority of cases where spectral clustering would be used. 
		\item The control of the approximation error comes in different flavors, that we detail here from the tightest to the loosest. The best possible error control in our context is a control over the clustering solution itself, via error measures such as the misclustering rate. This is unfortunately unrealistic in many cases. An excellent  alternative is the multiplicative error --considered as the gold standard in approximation theory-- over the $k$-means cost\footnote{A control in terms of the $k$-means cost is usually considered as $k$-means is the last step of spectral clustering. Nevertheless, recalling the minimum cut perspective of Section~\ref{subsec:graph_cut}, the control should arguably be in terms of \rcut\ or \ncut\ costs.}, ensuring that the cost of the approximation is not larger than $1+\epsilon$ times the cost of the exact solution. Next comes the additive error over the cost: ensuring that the cost difference between approximated and exact solutions is not larger than $\epsilon$. All these error controls are referred to as end-to-end controls, and represent the limit of what we will consider a \emph{satisfactory} error control.
	\end{itemize}
	\end{itemize}

	Reviewing the literature, we were surprised to discover that there are rarely any algorithms meeting fully the proposed criteria: a faster algorithm with end-to-end control over the approximation error under mild assumptions. Let us revisit one by one the different approaches presented in Sections~\ref{sec:from_scratch}, ~\ref{sec:spectral_embedding} and ~\ref{sec:kmeans} examining them in light of our criteria for success. In each category of approximation algorithms, we order the methods according to the power of their error control.\\
	
	\noindent\textbf{Sampling methods in the original feature space [Section~\ref{sec:from_scratch}].}
	\begin{itemize}
		\item \emph{Representative points methods} as described in~\cite{yan_fast_2009, NIPS2008_3480} allow for an end-to-end control on the miss-clustering rate $\rho$, which is unfortunately quite loose. The constants involved in Theorem~\ref{thm:KASP} are in fact undefined --thus potentially large--, which is problematic knowing that $\rho$ is by definition between $0$ an $1$. Also, the theorem's assumptions include independence of the $\epsilon_i$, which is hard to justify in practice. On the other hand, the computation gain of such methods is very appealing.
		\item \emph{Feature projection methods}, where the dimension $d$ of the original feature space is reduced to a dimension $d'\leq d$ based on Johnson-Lindenstrauss arguments, come with a multiplicative error control on the pairwise distances in the original feature space, thus providing a control on the obtained kernel matrix. The impact of this initial approximation on the final clustering result has not been studied. 
		\item \textit{Nystr\"om-inspired methods}~\cite{fowlkes_spectral_2004, li2011time,bouneffouf2015sampling,mohan2017exploiting} can be very efficient in practice especially because they do not need to build the graph. However, precisely because they do not build the graph, these methods cannot exactly perform two key parts of the prototypical spectral clustering algorithm: the $k$-NN sparsification and the exact degree computation. The partial knowledge and sparsity of the kernel matrix also makes sampling difficult, as using leverage scores sampling is not possible anymore, whereas most other sampling schemes do not work very well with sparse matrices and come with weak guarantees. To the extent of our knowledge, there is also no convincing mathematical argument proving that using these methods will yield a clustering that is of similar quality to that produced by the exact spectral clustering algorithm.  
		\item \emph{Sketching methods} such as the Random Fourier Features~\cite{rahimi_random_2008} is yet another way of obtaining a pointwise multiplicative $(1+\epsilon)$ error on the Gaussian kernel computation. RFF enable to compute a provably good low-rank approximation of the kernel. They nevertheless suffer from the same problems as Nystr\"om-based techniques: without building the graph, sparsification and degree-normalization are uncontrolled. In addition, the guarantees on the low-rank approximation of the kernel do not transfer easily to guarantees of approximation of the spectral embedding $U_k$. 
		\item \emph{Approximate nearest neighbours methods} are numerous and varied, and come with different levels of guarantees. Practical implementations of algorithms, however, often set aside theoretical guarantees to gain on efficiency and performances; and comparisons are usually done on benchmarks rather than on theoretical performances. In the best of cases, there is a control on how close the obtained nearest neighbour similarity graph is to the exact one; but with no end-to-end control. 
	\end{itemize}
	
	\noindent\textbf{Spectral embedding approximation methods [Section~\ref{sec:spectral_embedding}].}
	\begin{itemize}
	
		\item \textit{Random eigenspace projection} is a very fast method and has been rigorously analyzed~\cite{martin2018fast,tremblay2016compressive, paratte2016fast, ramasamy_compressive_2015}. It is true that a successful application depends on obtaining a good estimate of the $k$-th eigenvalue, which is very hard when the $k$-th eigenvalue gap is relatively small. Nevertheless, our current understanding of spectral clustering suggests that it only works well when the gap is (at least) moderately large. As such, though there are definitely situations in which random eigenspace projection will fail to provide an acceleration, these correspond to cases where one should not be using spectral clustering in the first place. The same argumentation can also be used in defense of all methods that come with mild gap assumptions (see coarsening, and spectral sparsification). 

		\item \textit{Simple coarsening methods}, such as the heavy-edge matching heuristic~\cite{karypis1998fast}, have nearly-linear complexity, seem to work well in practice, and are accompanied by end-to-end additive error control~\cite{loukas2018spectrally}. Nevertheless, the current analysis of these heuristics only accounts for very moderate reductions ($m \geq n/2$) and thus does not fully prove their success: in real implementations coarsening is used in a multi-level fashion resulting to a drastic decrease in the graph size ($m = \O(n/2^\c) $ for $\c$ levels), whereas the end-to-end control only works for a single level. 
		
		\item \textit{Advanced coarsening methods}, such as local variation methods~\cite{loukas2018graph}, come with much stronger guarantees that allow for drastic size reduction and acceleration. Yet, thus far, all evidence suggests that finding a good enough coarsening is computationally as hard as solving the spectral clustering problem itself. As a consequence, it is at this point unclear whether these methods can be used to accelerate spectral clustering. 
		
		\item \textit{Spectral sparsification techniques} come with excellent guarantees in theory: one may prove that a spectral sparsifier can be computed in nearly-linear time and, moreover, the latter's spectrum will be provably close to the original one. Yet, we have reasons to doubt their practicality. Indeed, current algorithms are very complex, feature impractically large constants, and are only relevant for dense graphs. In addition, spectral sparsifiers, by definition, approximate the entire spectrum of a graph Laplacian matrix. However, spectral clustering only needs an approximation of a tiny fraction of the spectrum. From that perspective, it is reasonable to conclude that without modification current approaches will not yield the best possible approximation.  
	    \item \textit{Nyström-approximation} applied directly to the Laplacian matrix is a good option, especially when combined with leverage score sampling. Nevertheless, an end-to-end error control has only been partially derived and is not yet satisfactory.	    
	\end{itemize}
	
	\noindent\textbf{Sampling to accelerate the $k$-means step [Section~\ref{sec:kmeans}].}
	\begin{itemize}
		\item \emph{Exact methods to accelerate the Lloyd-Max algorithm}, may they be via avoiding unnecessary distance calculations or via a careful initialization are always useful and should be taken into account.
		\item \emph{Coresets} come with the strongest guarantees: the minimum number of samples to guarantee a  $(1+\epsilon)$ multiplicative error on the cost function has been well studied. Nevertheless, practical coreset sampling methods are scarce; and in the best cases, the sampling cost is of the same order of the Lloyd-Max running cost itself.
		\item \emph{Graph-based sampling} comes with strong guarantees, but not over the $k$-means cost: on the reconstruction error based on a $k$-bandlimited model that is only an approximation in practice. Moreover, we interpret the reduced indicator vectors $z_\ell^r$ obtained by running Lloyd-Max on the sampled set $\sf{S}$ as (possibly noisy) measurements of $z_\ell$ on $\sf{S}$. 
		This interpretation currently lacks solid theoretical ground and impedes an end-to-end control of this approximation method. Nevertheless, the leverage-score based sampling allows for a reduction in order of magnitude of the Lloyd-Max running cost. 
		\item \emph{Other methods} to accelerate $k$-means are not always appropriate to the spectral clustering context. Spectral feature dimension reduction is unnecessary in our context where $d=k$, sketching methods appropriate to distributed cases where $n$ is very large are not appropriate neither as the spectral features need centralized data to be computed in any case. 
	\end{itemize}
	
	\noindent\textbf{In practice.} The attentive reader will have remarked that, unsurprisingly, the tighter the error control, the more expensive the computation, and vice versa. Also, although we have put here an emphasis on the approximation error controls, it should not undermine the fact that methods from the whole spectrum are in practice useful, depending on the situation at hand, and specifically on the range of values of $n$, $d$ and $k$. In very large $d$ situations, a first step of random projection (or feature selection if some features are suspected to be too noisy) should be considered. Then, in situations where the exact computation of the proximity graph is too expensive, one may resort either to sketching methods or to  Nystr\"om-type methods to decrease the cost from quadratic to linear in $n$, and directly obtain an approximation of the spectral embedding without any explicit graph construction. 
	These methods, however, do not take into account a sparsity constraint on the proximity graph and are usually rough on the degree correction they make. 
	
	The role of the sparsity constraint is not well understood theoretically, but seems to be important in some practical cases~\cite{von2007tutorial}. In such instances, a better option is to use approximate nearest neighbours methods to create a sparse similarity graph, and work from there. In extremely large data, say $n\geq 10^8$, the only workable methods are the representative-based, with, if possible, a first $k$-means (or compressive $k$-means~\cite{keriven_compressive_2017}) to reduce $n$ to $m$ ; or, in last resort, a uniform random sampling strategy. 
	
	In situations where one has to deal with such a large similarity graph that Arnoldi iterations are too expensive to compute the spectral embedding 
	(either a graph created via approximate nearest neighbours, or if the original data \emph{is} a graph), projection methods such as in~\cite{tremblay2016compressive, boutsidis2015spectral},  coarsening methods such as in~\cite{loukas2018graph}, or  Nyström-based methods are different possibilities. 
	
	Sampling methods to accelerate the last $k$-means step may seem to be a theoretical endeavour given that the Lloyd-Max algorithm 
	is already very efficient. Due to the quadratic term in $k$, it is nevertheless in practice useful when $k$ grows large. In this situation, hierarchical $k$-means~\cite{nister_scalable_2006} is a nice option. Coresets, because they are so stringent on the error control, have a hard time actually accelerating $k$-means, unless hybrid coreset-inspired methods are envisioned~\cite{frahling_fast_2008}. Finally, graph-based methods, because they take into account that spectral features are in fact derived from the graph itself, enable significant acceleration and are well-suited to the spectral clustering context. \\
	 
	\noindent\textbf{Future research.} Different directions of research could be envisioned to improve the state-of-the-art: 
	
	\begin{itemize}
	    \item 	For Nystr\"om-inspired methods in the context of Sec.~\ref{sec:from_scratch} (directly applied on the original data) as well as the other methods based on  computing a low-rank approximation of the kernel matrix $K$,  further work is needed to control both the sparsification and the degree correction, in order to bridge the gap between a provably good low-rank approximation of $K$ to a provably good low-rank approximation of  $R$.
	    
	    \item For Nystr\"om methods in the context of Sec.~\ref{sec:spectral_embedding} (applied on a known or well-approximated similarity graph), it would be interesting to extend Theorem~\ref{thm:lev_score_Nystrom} (for instance) to a control over $\|A_k - \tilde{A}_k\|$ instead of $\|A - \tilde{A}\|$. This would enable a tighter use of Davis-Kahan's perturbation theorem in the discussion of Sec.~\ref{subsec:Nystrom2} and, $\emph{in fine}$, a better end-to-end guarantee.
	
	    \item Projection-based methods of Sec.~\ref{subsec:projection-based} currently necessitate to compute a value $\lambda_*$ known to be in the interval $[\lambda_k, \lambda_{k+1})$. The algorithm used to do so is based on eigencount techniques that turn out to require as much computation time as the Lanczos iterations needed to compute $U_k$ exactly. One should relax this constraint to obtain end-to-end guarantees as a function of the distance between a  coarsely estimated $\lambda_*$ and the target interval.  
	
	    \item The derivation and analysis of randomized multi-level coarsening schemes with end-to-end guarantees is very much an open problem. We suspect that, by utilizing spectrum-dependent sampling-schemes akin to leverage-scores one should be able to achieve results superior to heavy edge-matching in nearly linear time.     
	    
	    \item There is an interesting similarity between coreset techniques and the graph-based sampling strategies discussed in Sec.~\ref{subsec:graph_based} and it would be interesting to investigate this link theoretically, maybe paving the way to coresets for spectral clustering?
	
	\end{itemize}

	Finally, accelerating the prototypical spectral algorithm depicted in Algorithm 1 should not be the sole objective of researchers in this field. Indeed, taking the graph cut point-of-view of Sec.~\ref{subsec:graph_cut}, Algorithm 1 makes three insufficiently motivated choices: (i) To begin with, the sparsification step in Algorithm 1 is not well understood. Apart from the fact that it is always computationally more convenient to work with a sparse similarity graph then a dense one, the precise effect of sparsification on the clustering performance has not been analyzed.
	(ii) As mentioned in Section~\ref{subsubsec:relax_bro}, the relaxation employed by spectral relaxation is not unique. Why should we focus our attention on this one versus another? See for instance~\cite{bresson_multiclass_2013, rangapuram_tight_2014} for recent alternative options. 
	(iii) Finally, the use of $k$-means on the spectral features is not yet fully justified.  Most of the end-to-end guarantees presented here compare the $k$-means cost of the exact solution to the $k$-means cost of the approximate solution. Given that the very use of $k$-means is not theoretically grounded, this choice of guarantee is debatable. Other options, such as a control over the \rcut\ or \ncut\ objectives are possible (as in~\cite{Mohan2017BeyondTN}) and should be further investigated.
	
	\section*{Acknowledgments}
	
	This work was partially supported by the CNRS PEPS I3A (Project RW4SPEC) and the Swiss National Science Foundation (Project DLEG, grant PZ00P2 179981). 
	
	
	

	{\footnotesize
	\bibliographystyle{abbrv}
	\bibliography{references.bib} }
	
\end{document}